\documentclass[english,11pt]{article}
\usepackage[T1]{fontenc}
\usepackage[latin9]{inputenc}
\usepackage{geometry}
\usepackage{babel}
\usepackage{amsthm}
\usepackage{amsmath}
\usepackage{amssymb}
\usepackage[unicode= true,pdfusetitle,
 bookmarks= true,bookmarksnumbered= false,bookmarksopen= false,
 breaklinks= false,pdfborder= {0 0 1},backref= false,colorlinks= false]
 {hyperref}
\usepackage{breakurl, mathrsfs, appendix}
\usepackage{fancyhdr}
\usepackage{fullpage}
\usepackage{graphicx,caption,subcaption}

\usepackage[firstinits=false, backend=bibtex, maxbibnames=99]{biblatex}
\makeatletter

\usepackage{bm}
\usepackage{dsfont}

\newlength{\widebarargwidth}
\newlength{\widebarargheight}
\newlength{\widebarargdepth}

\newcommand{\Prob}{\ensuremath{\mathbb{P}}}
\newcommand{\Exs}{\ensuremath{\mathbb{E}}}

\newcommand{\order}{\ensuremath{\mathcal{O}}}
\newcommand{\apporder}{\ensuremath{\widetilde{\mathcal{O}}}}
\newcommand{\real}{\ensuremath{\mathbb{R}}}
\newcommand{\sphere}{\ensuremath{\mathbb{S}}}

\newcommand{\argmin}{\arg\!\min}

\newcommand{\inprod}[2]{\ensuremath{\langle #1 , \, #2 \rangle}}

\DeclareMathOperator*{\ind}{\mathds{1}}  

\newcommand{\abs}[1]{| #1 |}

\newcommand{\opnorm}[1]{\ensuremath{\left\|#1\right\|_{\tiny{op}}}}
\newcommand{\twonorm}[1]{\ensuremath{\left\|#1\right\|_{2}}}
\newcommand{\Gaussian}{\mathcal{N}}

\newcommand{\Loss}{\ensuremath{\mathcal{L}}}

\newcommand{\ab}{\bm{a}}
\newcommand{\bb}{\bm{b}}
\newcommand{\cb}{\bm{c}}
\newcommand{\x}{\bm{x}}
\newcommand{\w}{\bm{w}}
\newcommand{\z}{\bm{z}}
\newcommand{\ub}{\bm{u}}
\newcommand{\vb}{\bm{v}}
\newcommand{\Identity}{\bm{I}}
\newcommand{\e}{\bm{e}}
\newcommand{\M}{\bm{M}}
\newcommand{\m}{\bm{m}}
\newcommand{\W}{\bm{W}}
\newcommand{\Q}{\bm{Q}}
\newcommand{\A}{\bm{A}}
\newcommand{\T}{\bm{T}}
\newcommand{\E}{\bm{E}}
\newcommand{\Ub}{\bm{U}}
\newcommand{\Vb}{\bm{V}}
\newcommand{\B}{\bm{B}}
\newcommand{\D}{\bm{D}}
\newcommand{\Cb}{\bm{C}}

\newcommand{\sign}{\text{sign}}
\newcommand{\Sb}{\bm{S}}

\newcommand{\minprop}{\underline{\omega}}

\newcommand{\estbeta}{\widehat{\bbeta}}

\newcommand{\initsample}{n_{\text{init}}}
\newcommand{\altsample}{n_{\text{alt}}}
\newcommand{\cA}{\mathcal{A}}

\newcommand{\bbeta}{\bm{\beta}}
\newcommand{\balpha}{\bm{\alpha}}
\newcommand{\bgamma}{\bm{\gamma}}
\newcommand{\bSigma}{\bm{\Sigma}}

\usepackage{algorithmicx}
\usepackage{algorithm}
\usepackage{algpseudocode}
\algnewcommand\algorithmicinput{\textbf{INPUT:}}
\algnewcommand\INPUT{\item[\algorithmicinput]}
\algnewcommand\algorithmicoutput{\textbf{OUTPUT:}}
\algnewcommand\OUTPUT{\item[\algorithmicoutput]}
\algnewcommand\algorithmicremark{\textsl{//}}
\algnewcommand\REMARK{\item[\algorithmicremark]}

\theoremstyle{plain}
\newtheorem{theorem}{\protect\theoremname}
\theoremstyle{remark}
\newtheorem{remark}{\protect\remarkname}
\theoremstyle{plain}
\newtheorem{lemma}{\protect\lemmaname}
\theoremstyle{plain}

\theoremstyle{plain}
\newtheorem{corollary}{\protect\corolaryname}
\theoremstyle{plain}
\newtheorem{condition}{\protect\conditionname}

\providecommand{\lemmaname}{Lemma}
\providecommand{\remarkname}{Remark}
\providecommand{\theoremname}{Theorem}
\providecommand{\propositionname}{Proposition}
\providecommand{\corolaryname}{Corollary}
\providecommand{\conditionname}{Condition}

\makeatletter
\long\def\@makecaption#1#2{
        \vskip 0.8ex
        \setbox\@tempboxa\hbox{\small {\bf #1:} #2}
        \parindent 1.5em  
        \dimen0=\hsize
	       \advance\dimen0 by -3em
        \ifdim \wd\@tempboxa >\dimen0
                \hbox to \hsize{
                        \parindent 0em
                        \hfil 
                        \parbox{\dimen0}{\def\baselinestretch{0.96}\small
                                {\bf #1.} #2
                                } 
                        \hfil}
        \else \hbox to \hsize{\hfil \box\@tempboxa \hfil}
        \fi
        }
\makeatother

\begin{document}

\begin{center} {\LARGE{\bf{Solving a Mixture of Many Random Linear Equations by Tensor Decomposition and Alternating Minimization
}}} \\

\vspace*{.3in}

{\large{
\begin{tabular}{ccccccc}
Xinyang Yi && Constantine Caramanis && Sujay Sanghavi
\end{tabular}

\vspace*{.1in}

\begin{tabular}{c}
The University of Texas at Austin
\end{tabular}

\begin{tabular}{c}
{\texttt{$\{$yixy,constantine$\}$@utexas.edu $\quad$ sanghavi@mail.utexas.edu}}
\end{tabular} 
}}

\vspace*{.2in}


\vspace*{.2in}
\begin{abstract}
We consider the problem of solving mixed random linear equations with $k$ components. This is the noiseless setting of mixed linear regression. The goal is to estimate multiple linear models from mixed samples in the case where the labels (which sample corresponds to which model) are not observed. We give a tractable algorithm for the mixed linear equation problem, and show that under some technical conditions, our algorithm is guaranteed to solve the problem exactly with sample complexity linear in the dimension, and polynomial in $k$, the number of components. Previous approaches have required either exponential dependence on $k$, or super-linear dependence on the dimension. The proposed algorithm is a combination of tensor decomposition and alternating minimization. Our analysis involves proving that the initialization provided by the tensor method allows alternating minimization, which is equivalent to EM in our setting, to converge to the global optimum at a linear rate.
\end{abstract}
\end{center}


\section{Introduction}
In this paper, we consider the following mixed linear equation problem. Suppose we are given $n$ samples of response-covariate pairs $\{(y_i, \x_i)\}_{i=1}^n$ that are determined by equations
\begin{equation} \label{eq:eq_xy}
y_i = \sum_{j=1}^k \inprod{\x_i}{\bbeta_j} \ind(z_i = j), ~\text{for}~ i = 1,\ldots,n,
\end{equation}
where $\x_i, \bbeta_j \in \real^p$, $\{\bbeta_j\}$ are $k$ model parameters corresponding to $k$ different linear models, and $z_i$ is the {\em unobserved} label of sample $i$ indicating which model it is generated from. We assume random label assignment, i.e., $\{z_i\}$ are i.i.d. copies of a multinomial random variable $Z$ that has distribution
\begin{equation} \label{eq:gen_label}
\Prob[Z = j] = \omega_j, ~\text{for}~ j = 1,2,\ldots k.
\end{equation}
Here $\{\omega_j\}$ represent the weights of every linear model, and naturally satisfy $\sum_{j \in [k]}\omega_j = 1$. Our goal is to find parameters $\{\bbeta_j\}$ from mixed samples $\{(y_i, \x_i)\}_{i=1}^n$. While solving linear systems is straightforward, this problem, with the introduction of latent variables, is hard to solve in general. Work in \cite{yi2014alternating} shows that the subset sum problem can be reduced to mixed linear equations in the case of $k = 2$ and certain designs of $\x_i$ and $\bbeta_j$. Therefore, given $\{(y_i, \x_i)\}_{i=1}^n$, determining whether there exist two $\bbeta$'s that satisfy $\eqref{eq:eq_xy}$ is NP-complete, and thus in general the $k=2$ case is already hard. In this paper, we consider the setting for general $k$, where the covariates $\x_i$'s are independently drawn from the standard Gaussian distribution:
\begin{equation} \label{eq:x_dist}
\x_i \sim \Gaussian(\bm{0}, \Identity_p).
\end{equation}
Under this random design, we provide a tractable algorithm for the mixed linear equation problem, and give sufficient conditions on the $\bbeta_i$'s under which we guarantee exact recovery with high probability.

The problem of solving mixed linear equations (or regression when each $y_i$ is perturbed by a small amount of noise) arises in applications where samples are from a mixture of discriminative linear models and the interest is in parameter estimation. Mixed standard and generalized linear regression models are introduced in 1990s \cite{wedel1995mixture} and have become an important set of techniques for market segmentation \cite{wedel2012market}. These models have also been applied to study music perception \cite{Viele2002}  and health care demand \cite{deb2000estimates}. See \cite{grun2007applications} for other related applications and datasets. Mixed linear regression is closely related to another classical model called hierarchical mixtures of experts \cite{jacobs1991adaptive}, which also allows the distribution of labels to be adaptive to covariate vectors.

Due to the combinatorial nature of mixture models, popular approaches, including EM and gradient descent, are often based on local optimization and thus suffer from local minima. Indeed, to the best of our knowledge, there is no rigorous analysis of the convergence behavior of EM or other gradient descent-based methods for $k \geq 3$. Beyond real-world applications, the statistical limits of solving problem \eqref{eq:eq_xy} by computationally efficient algorithms are even less well understood. This paper is motivated by this question: how many samples are necessary to recover $\{\bbeta_j\}$ exactly and efficiently?

In a nutshell, we prove that under certain technical conditions, there exists an efficient algorithm for solving mixed linear equations with sample size $\widetilde{\order}(k^{10}p)$, and we provide an algorithm which achieves this. Notably, the dependence on $p$ is nearly linear and thus optimal up to some log factors. Our proposed algorithm has two phases. The first step is a spectral method called tensor decomposition, which is guaranteed to produce $\varepsilon$-close solutions with $\order(1/\varepsilon^2)$ samples. In the second step, we apply an alternating minimization (AltMin) procedure to successively refine the estimation until exact recovery happens. As a key ingredient, we show that AltMin, as a non-convex optimization technique, enjoys linear convergence to the global optima when initialized closely enough to the optimal solution.

\subsection{Comparison to Prior Art}

The use of the method of moments for learning latent variable models can be dated back to Pearson's work \cite{pearson1894contributions} on estimating Gaussian mixtures. There is now an increasing interest in computing high order moments and leveraging tensor decomposition for parameter estimation in various mixture models including Hidden Markov Models \cite{anandkumar2012method}, Gaussian mixtures \cite{Hsu2012Gussian}, and topic models \cite{anandkumar2012two}. Following the same idea, we propose some novel moments for mixed linear equations, on which approximate estimation of parameters can be computed by tensor decomposition. Different from our moments, Chaganty and Liang \cite{chaganty13} propose a method of regressing $y_i^3$ against $\x_i^{\otimes3}$ to estimate a certain third-moment tensor of mixed linear regression under bounded and random covariates. Because of performing regression in the lifted space with dimension $p^3$, their method suffers from much higher sample complexity $\order(p^6)$ compared to our results, while the latter builds on a different covariate assumption \eqref{eq:x_dist}.

Mixed linear equation/regression with two components is now well understood. In particular, our earlier work \cite{yi2014alternating} proves the local convergence of AltMin for mixed linear equations with two components. Through a convex optimization formula, work in \cite{chen2014convex} establishes the minimax optimal statistical rates under stochastic and deterministic noises. Notably, Balakrishnan et~al.\ \cite{balakrishnan2014statistical} develop a framework for analyzing the local convergence of expectation-maximization (EM), i.e., EM is guaranteed to produce statistically consistent points with good initializations. In the case of mixed linear regression with $\bbeta_1 = -\bbeta_2$ and Gaussian noise with variance $\sigma^2$, applying the framework leads to estimation error $\widetilde{\order}(\sqrt{(\sigma^2 + \twonorm{\bbeta_1})p/n})$. Even in the case of no noise ($\sigma = 0$), their results do not imply exact recovery. Moreover, it is unclear how to apply the framework to the case of $k \geq 3$ components. It is obvious that AltMin is equivalent to EM in the noiseless setting. Our analysis of AltMin takes a step further towards understanding EM in the case of multiple latent clusters.

Beyond linear models, learning mixture of generalized linear models is recently studied in \cite{sun2014learning} and \cite{sedghi2014provable}. Specifically, \cite{sun2014learning} proposes a spectral method based on second order moments for estimating the subspace spanned by model parameters. Later on, Sedghi et~al.\ \cite{sedghi2014provable} construct specific third order moments that allow tensor decomposition to be applied to estimate individual vectors. In detail, when $k = \order(1)$, they show that obtaining recovery error $\varepsilon$ requires sample size $n = \widetilde{\order}(p^3/\varepsilon^2)$. In a more recent update \cite{sedghi2016provable} of their paper, they establish the same sample complexity for mixed linear regression using different moments, which we realize coincide with ours during the preparation of this paper. Nevertheless, we perform a sharper analysis that leads to a near-linear-in-$p$ sample complexity $n = \widetilde{\order}(p/\varepsilon^2)$.

Conceptually, we establish the power of combining spectral method and likelihood based estimation for learning latent variable models. Spectral method excludes most bad local optima on the surface of likelihood loss, and as a consequence, it becomes much easier for non-convex local search methods such as EM and AltMin, to find statistically efficient solutions. Such phenomenon in the context of mixed linear regression is observed empirically in \cite{chaganty13}. We provide a theoretical explanation in this paper. It is worth mentioning the applications of such idea in other problems including crowdsourcing \cite{zhang2014spectral}, phase retrieval (e.g. \cite{candes2015phase, chen2015solving}) and matrix completion (e.g. \cite{jain2013low, sun2015guaranteed, chen2015fast}). Most of these works focus on estimating bilinear or low rank structures. In the context of crowdsourcing, work in \cite{zhang2014spectral} shows that performing one step of EM can achieve optimal rate given good initialization.  In contrast, we establish an explicit convergence trajectory of multiple steps of AltMin for our problem. It would be interesting to study the convergence path of AltMin or EM for other latent variable models.

\subsection{Notation and Outline}
We lay down some notations commonly used throughout this paper. For counting number $k$, we use $[k]$ to denote the set $\{1,2,\ldots,k\}$. We let $a \vee b$, $a \wedge b$ denote $\max\{a,b\}, \min\{a,b\}$ respectively. For sub-Gaussian random variable $X$, we denote its $\psi_2$-Orlicz norm \cite{van1996weak} by $\|X\|_{\psi_2}$, i.e.,
\[
\|X\|_{\psi_2} := \inf\left\{z \in (0, \infty) ~\big|~ \Exs[\psi_2(\abs{X}/z)] \leq 1 \right\},	
\]
where $\psi_{2}(x) = \exp(x^2) - 1$.
For vector $\ab \in \real^p$, we use $\|\ab\|_q$ to denote the standard $\ell_q$ norm of $\ab$. For matrix $\A \in \real^{p_1\times p_2}$, we use $\sigma_{k}(\A)$ to denote its $k$-th largest singular value. We also commonly use $\sigma_{\max}(\A), \sigma_{\min}(\A)$ to denote $\sigma_{1}(\A)$ and $\sigma_{p_1 \wedge p_2}(\A)$. In particular, we denote the operator norm of matrix $\A$ as $\opnorm{\A}$. We also use $\opnorm{\T}$ to denote the operator norm of symmetric third order tensor $\T \in \real^{p \times p \times p}$, namely
\[
\opnorm{\T} := \sup_{\ub \in \sphere^{p-1}} \abs{\T(\ub, \ub, \ub)}.
\]
Here, $\T(\A, \B, \Cb)$ denotes the multi-linear
matrix multiplication of $\T$ by $\A \in \real^{p\times p_1}, \B \in \real^{p \times p_2}, \Cb \in \real^{p \times p_3}$, namely,
\[
(\T(\A, \B, \Cb))_{(m,n,t)} = \sum_{i,j,k \in [p]}\T_{(i,j,k)}\A_{(i,m)}\B_{(j,n)}\Cb_{(k,t)}, ~~\text{for all}~~ (m,n,t) \in [p_1]\times[p_2]\times [p_3].
\]

For two sequences $f(n), g(n)$ indexed by $n \in \mathbb{N}$, we write $f(n) = \order(g(n))$ to mean there exists a constant $C > 0$ such that $f(n) \leq C g(n)$ for all $n \in \mathbb{N}$. By $f(n) = \widetilde{\order}(g(n))$, we mean there exist constants $C,C' > 0$ such that $f(n) \leq Cg(n)\cdot(\log n)^{C'}$.  We also use $f(n) \lesssim g(n)$ as shorthand for $f(n) = \order(g(n))$. Similarly, we say $f(n) \gtrsim g(n)$ if $g(n) = \order(f(n))$.

The rest of this paper is organized as follows. In Section \ref{sec:algorithm}, we describe the specific details of our two-phase algorithm for solving mixed linear equations. We present the theoretical results of initialization and AltMin in Section \ref{sec:tensor} and \ref{sec:altmin} respectively. We combine these two parts and give the overall sample and time complexities for exact recovery in Section \ref{sec:overall}. We provide the experimental results in Section \ref{sec:numeric}. All proofs are collected in Section \ref{sec:proofs}.

\section{Algorithm} \label{sec:algorithm}
A natural idea to solve problem \eqref{eq:eq_xy} is to apply an alternating minimization (AltMin) procedure between parameters $\{\bbeta_j\}$ and labels $\{z_i\}$: (1) Given $\{\bbeta_j\}$, assign the labels for each sample by choosing a model $\bbeta$ that has minimal recovery error $\abs{y_i - \inprod{\x_i}{\bbeta}}$; (2) When labels are available, each parameter is updated by applying the method of least square optimization to samples with the corresponding labels. One can show that in our setting, alternating minimization is equivalent to Expectation-Maximization (EM), which is one of the most important algorithms for inference in latent variable models. In general, similar to EM, AltMin is vulnerable to local optima. Our experiment (see Figure \ref{fig:trace}) demonstrates that even under random setting $\x_i \sim \Gaussian(\bm{0}, \Identity_p)$, AltMin with random initializations fails to exactly recover each $\bbeta_j$ with significantly large probability.

To overcome the local-optima issue of AltMin, our algorithm consists of two stages. The first stage builds on carefully designed moments of samples, and aims to find rough estimates of $\{\bbeta_j\}$. Starting with the initialization, the second stage involves using AltMin to successively refine the estimates. In the following, we describe these two steps with more details.
\subsection{Tensor Decomposition}
In the first step, we use method of moments to compute initial estimates of $\{\bbeta_j\}$. Consider moments $m_0 \in \real, \m_1 \in \real^p$, $\M_2 \in \real^{p\times p}$ and $\M_3 \in \real^{p\times p \times p}$ as
\begin{align}
m_0 & := \frac{1}{n}\sum_{i=1}^n y_i^2, ~~ \m_1 := \frac{1}{6n}\sum_{i = 1}^n y_i^3\x_i,  \label{eq:m_01}\\
\M_2 & := \frac{1}{2n}\sum_{i=1}^{n} y_i^2 \x_i \otimes \x_i - \frac{1}{2}m_0\cdot \Identity_p, \label{eq:m_2}\\
\M_3 & := \frac{1}{6n}\sum_{i=1}^{n} y_i^3 \x_i\otimes \x_i \otimes \x_i - \mathcal{T}(\m_1), \label{eq:m_3}
\end{align}
where $\mathcal{T}(\cdot)$ is a mapping from $\real^{p}$ to $\real^{p \times p \times p}$ with form
\[
\mathcal{T}(\m_1) := \sum_{i \in [p]} \m_1 \otimes \e_i \otimes \e_i + \e_i \otimes \m_1\otimes \e_i +  \e_i \otimes \e_i \otimes \m_1.
\]
It is reasonable to choose these moments because of the next result, which shows that the expectations of $M_2$ and $M_3$ contain the structure of $\{\bbeta_j\}$. See Section \ref{sec:proof:lem:expectation} for its proof.
\begin{lemma}[Moment Expectation] \label{lem:expectation}
	Consider the random model for mixed linear equations given in \eqref{eq:eq_xy}, \eqref{eq:gen_label} and \eqref{eq:x_dist}. For moments $\M_2$ and $\M_3$ in \eqref{eq:m_2} and \eqref{eq:m_3}, we have
	\begin{align}
	\Exs[\M_2] & = \sum_{j=1}^k \omega_j \cdot \bbeta_j \otimes \bbeta_j, \label{eq:M_2_exp}\\
	\Exs[\M_3] & = \sum_{j=1}^k \omega_j \cdot \bbeta_j \otimes \bbeta_j \otimes \bbeta_j. \label{eq:M_3_exp}
	\end{align}
\end{lemma}

With the special structure given on the right hand sides of \eqref{eq:M_2_exp} and \eqref{eq:M_3_exp}, tensor decomposition techniques can discover $\{(\omega_j, \bbeta_j)\}$ in three steps under a {\em non-degeneracy condition} (see Condition \ref{cond:non_degeneracy}). First, apply SVD on $\Exs[\M_2]$ to compute a whitening matrix $\W \in \real^{p \times k}$ such that $\W^\top \Exs[\M_2] \W = \Identity_p$. Then we use $\W$ to transform $\Exs[\M_3]$ into an orthogonal tensor $\Exs[\M_3](\W, \W, \W)$, which is further decomposed into eigenvalue/eigenvector pairs by robust tensor power method (Algorithm \ref{alg:t_power}). Lastly, $\{(\omega_j, \bbeta_j)\}$ can be reconstructed by applying simple linear transformation upon the previously discovered spectral components from $\Exs[\M_3](\W, \W, \W)$. With sufficient amount of samples, it is reasonable to believe that $\M_2$ and $\M_3$ are close to their expectations such that the stability of tensor decomposition will lead to good enough estimates. For the ease of analysis, we need to ensure the independence between whitening matrix $\W$ and $\M_3$. Accordingly, we split the samples used in initialization into two disjoint parts for computing $\{m_0, \M_2\}$ and $\{\m_1, \M_3\}$ respectively. We present the details in Algorithm \ref{alg:tensor}.

\begin{algorithm}[H]
	\caption{Initialization via Tensor Factorization}
	\label{alg:tensor}
	\begin{algorithmic}[1]
		\INPUT Samples $\{(y_i, \x_i)\}_{i=1}^n$.
		\vskip .05in
		\State Randomly split samples into two disjoint parts $\{(y_i, \x_i)\}_{i=1}^{n_1}$ and $\{(y_i', \x_i')\}_{i=1}^{n_2}$.
		\State $m_0 \leftarrow \frac{1}{n_1}\sum_{i=1}^{n_1} y_i^2$, ~~$\m_1 \leftarrow \frac{1}{6n_2}\sum_{i = 1}^{n_2} y_i'^3\x_i'$.
		\State $\M_2 \leftarrow \frac{1}{2n_1}\sum_{i=1}^{n_1} y_i^2 \x_i \otimes \x_i - \frac{1}{2}m_0\cdot \Identity_p$,~~ $\M_3 \leftarrow \frac{1}{6n_2}\sum_{i=1}^{n_2} y_i'^3 \x_i'\otimes \x_i' \otimes \x_i' - \mathcal{T}(\m_1)$.
		\State Compute an SVD of the best rank $k$ approximation of $\M_2$ as $\Ub \bSigma \Ub^{\top}$, where $\Ub \in \real^{p \times k}$. Compute whitening matrix $\W \leftarrow \Ub \bSigma^{-1/2}$.
		\State $\widetilde{\M}_3 \leftarrow \M_3(\W, \W,\W)$. 
		\State Run robust tensor power method (Algorithm \ref{alg:t_power}) on $\widetilde{\M}_3$ to obtain $k$ eigenvalue/eigenvector pairs $\{(\widetilde{\omega}_j, \widetilde{\bbeta}_j)\}_{j=1}^k$.
		\State $\omega_j^{(0)} \leftarrow 1/\widetilde{\omega}_j^2$, $\bbeta_j^{(0)} \leftarrow \widetilde{\omega}_j(\W^{\top})^{\dagger}\widetilde{\bbeta}_j$, for all $j \in [k]$.
		\footnotemark
		\OUTPUT $\{(\omega_j^{(0)}, \bbeta_j^{(0)})\}_{j=1}^k$.
	\end{algorithmic}
\end{algorithm} 
\footnotetext{$(\W^{\top})^{\dagger}$ denotes the Moore-Penrose pseudoinverse of $\W^{\top}$, i.e., $\W(\W^{\top}\W)^{-1}$.}

\begin{algorithm}[H] 
	\caption{Robust Tensor Power Method (Algorithm 1 in \cite{anandkumar2014tensor})} 
	\label{alg:t_power}
	\begin{algorithmic}[1]
		\INPUT Symmetric tensor $\T \in \real^{k\times k\times k}$. Parameters $L, N$.
		\vskip .05in
		\For{$j = 1,\ldots,k$}
		\For{$l = 1,\ldots,L$}
		\State Draw $\bbeta_0^{(l)}$ uniformly at random from $\sphere^{k-1}$.
		\For{$t = 0,\ldots,N-1$}
		\begin{equation} \label{eq:tensor_power_iter}
		\bbeta_{t+1}^{(l)} \leftarrow \T(\Identity_k, \bbeta_t^{(l)}, \bbeta_t^{(l)}),~~~~ \bbeta_{t+1}^{(l)} \leftarrow \bbeta_{t+1}^{(l)}/\twonorm{\bbeta_{t+1}^{(l)}}.
		\end{equation}
		\EndFor
		\EndFor
		\State $l^* \leftarrow \arg\max_{l \in [L]} \T(\bbeta_{N}^{(l)}, \bbeta_{N}^{(l)},\bbeta_{N}^{(l)})$.
		\State Do $N$ power updates \eqref{eq:tensor_power_iter} starting from $\bbeta_{N}^{(l^*)}$ to obtain $\widetilde{\bbeta}_j$. Let $\widetilde{\omega}_j \leftarrow \T(\widetilde{\bbeta}_j,\widetilde{\bbeta}_j,\widetilde{\bbeta}_j)$.
		\State $T \leftarrow T - \widetilde{\omega}_j\widetilde{\bbeta}_j^{\otimes 3}$.
		\EndFor
		\OUTPUT $\{(\widetilde{\omega}_j, \widetilde{\bbeta}_j)\}_{j = 1}^k$.
	\end{algorithmic}
\end{algorithm}
	
\subsection{Alternating Minimization}
The motivation for using AltMin is to consider the least-square loss function below
\[
\Loss_n(\{\bbeta_j\}) := \min_{z_1,\ldots,z_n \in [k]} \sum_{i=1}^n \sum_{j = 1}^k (y_i - \inprod{\x_i}{\bbeta_j})^2 \ind(z_i = j) .
\]
The minimization over discrete labels $\{z_i\}$ makes the above loss function non-convex and yields hardness of solving mixed linear equations in general. A natural idea to minimize $\Loss_n$ is by minimizing $\{z_i\}$ and $\{\bbeta_j\}$ alternatively and iteratively. Given initial estimates $\{\bbeta_j^{(0)}\}$, each iteration $t = 0,1,\ldots$ consists of the following two steps:
\begin{itemize}
	\item {\bf Label Assignment:} Pick the model that has the smallest reconstruction error for each sample
	\begin{equation} \label{eq:label_assignment}
		z_i^{(t)} = \argmin_{j \in [k]} |y_i - \inprod{\x_i}{\bbeta_j^{(t)}}|.
	\end{equation}
	\item {\bf Parameter Update:}
	\begin{equation} \label{eq:parameter_update}
		\bbeta_j^{(t+1)} = \argmin_{\bbeta \in \real^p} \sum_{i=1}^n (y_i - \inprod{\x_i}{\bbeta})^2\ind(z_i^{(t)} = j).
	\end{equation}
\end{itemize}
AltMin runs quickly and is thus favored in practice. However, as we discussed before, its convergence to global optima is commonly intractable. In order to alleviate such issue, we already discussed how to construct good initial estimates by method of moments. Here, we introduce another ingredient---{\em resampling}---for making the analysis of AltMin  tractable. The key idea is to split all samples into multiple disjoint subsets and use a fresh piece of samples in each iteration. While slightly inefficient regarding sample complexity, this trick decouples the probabilistic dependence between two successive estimates $\{\bbeta_j^{(t)}\}$ and $\{\bbeta_j^{(t+1)}\}$, and thus makes our analysis hold. The details are presented in Algorithm \ref{alg:altmin}.

\begin{algorithm}[H]
	\caption{Alternating Minimization with Resampling}
	\label{alg:altmin}
	\begin{algorithmic}[1]
		\INPUT Samples $\{(y_i, \x_i)\}_{i=1}^n$, initial estimates $\{\bbeta_j^{(0)}\}$, number of iterations $T$.
		\vskip .05in
		\State Split all samples into $T$ disjoint subsets $\{(y_i^{(t)}, \x_i^{(t)})\}_{i=1}^{n/T}, t = 0,1,\ldots, T-1$, with equal size.
		\For{$t = 0, 1,\ldots,T-1$}
		\State \[
				z_i^{(t)} \leftarrow \argmin_{j \in [k]} |y_i^{(t)} - \inprod{\x_i^{(t)}}{\bbeta_j^{(t)}}|, ~\text{for all}~ i \in [n].
				\]
		\State \[
				\bbeta_j^{(t+1)} \leftarrow \argmin_{\bbeta \in \real^p} \sum_{i=1}^{n/T} (y_i^{(t)} - \inprod{\x_i^{(t)}}{\bbeta})^2\ind(z_i^{(t)} = j), ~\text{for all}~ j \in [k].
				\]
		\EndFor
		\OUTPUT $\{\bbeta_j^{(T)}\}_{j=1}^k$.
	\end{algorithmic}
\end{algorithm}
\section{Theoretical Results} \label{sec:theory}
In this section, we provide the theoretical guarantees of Algorithm \ref{alg:tensor} and \ref{alg:altmin}. For simplicity, we assume the $\ell_2$ norm of $\bbeta_j$ is at most $1$, i.e.,
\[
\max_{j \in [k]}\twonorm{\bbeta_j} = 1.
\]
Moreover, we impose the following non-degeneracy condition on $\{\bbeta_j\}$.
\begin{condition}[Non-degeneracy] \label{cond:non_degeneracy}Parameters $\bbeta_1,\ldots,\bbeta_{k}$ are linearly independent and all weights $\omega_j$ are strictly greater than $0$, namely
\[
\minprop := \min_{j \in [k]}~ \omega_{j} > 0.
\]
\end{condition} 
Under the above condition, $\overline{\M}_2 = \sum_{j \in [k]}\omega_j\bbeta_j\otimes\bbeta_j$ has rank $k$, which leads to
\[
\sigma_k := \sigma_k(\overline{\M}_2) > 0.
\]
We use $\Delta$ to denote the minimum distance between any two parameters, namely
\[
\Delta := \min_{i,j \in [k], i \ne j} \twonorm{\bbeta_i - \bbeta_j}.
\]
The above three quantities $(\minprop, \sigma_k, \Delta)$ represent the hardness of our problem, and will appear in the results of our analysis. For estimates $\{\estbeta_j\}$, we define the estimation error $\mathcal{E}(\{\estbeta_j\})$ as
\begin{equation}\label{eq:error}
\mathcal{E}(\{\estbeta_j\}) := \inf_{\pi}\sup_{j \in [k]}\twonorm{\estbeta_j - \bbeta_{\pi(j)}},
\end{equation}
where the infimum is taken over all permutations $\pi(\cdot)$ on $[k]$.

\subsection{Analysis of Tensor Decomposition} \label{sec:tensor}
Our first result, proved in Section \ref{sec:proof:thm:tensor}, provides a guarantee of Algorithm \ref{alg:tensor}.
\begin{theorem}[Tensor Decomposition] \label{thm:tensor} Consider Algorithm \ref{alg:tensor} for initial estimation of $\{\bbeta_j\}$. Pick any $\delta \in (0,1)$. There exist constants $C_i$ such that the following holds. Pick any $\varepsilon \in (0, C_1/k)$. If
\begin{equation} \label{eq:n_1}
n_1 \geq C_2\left(\frac{p\log(12k/\delta)\log^2n_1}{\minprop\sigma_k^5\varepsilon^2} 
\vee \frac{k}{\minprop \delta}\right) ~~\text{and}~~ n_2 \geq C_3\left(\frac{(k^2 \vee p)\log(12k/\delta)\log^3n_2}{\minprop\sigma_k^3\varepsilon^2}\vee \frac{k}{\minprop \delta}\right),
\end{equation}
then with probability at least $1 - \delta$, the output $\{\bbeta_j^{(0)}\}$ satisfy
\[
\mathcal{E}(\{\bbeta_j^{(0)}\}) \leq \varepsilon.
\] 
\end{theorem}
Theorem \ref{thm:tensor} shows that $n_1,n_2$ have inverse dependencies on $\minprop, \sigma_k$. In the well balanced setting, we have $\minprop = \Omega(1/k)$. In general, $\sigma_k$ can be quite small, especially in the case where some parameter $\bbeta$ almost lies in the subspace spanned by the rest $k-1$ parameters and has a very small magnitude along the orthogonal direction. Below we provide a sufficient condition under which $\sigma_k$ has a well established lower bound.
\begin{condition}[Nearly Orthonormal Condition($\eta,\gamma$)] \label{cond:approx_orthonormal} For all $j \in [k]$, $\twonorm{\bbeta_j} \geq 1 - \eta$. Moreover, $\abs{\inprod{\bbeta_i}{\bbeta_j}} \leq \gamma$ for all $i,j \in [k], i \ne j$.
\end{condition}
Under the above condition, the next result provides a lower bound of $\sigma_k$. See Section \ref{sec:proof:lem:lower_bound_sigma} for the proof.
\begin{lemma} \label{lem:lower_bound_sigma}Suppose $\{\bbeta_j\}$ satisfy the nearly orthonormal condition with $\eta, \gamma$. Then we have
	\[
	\sigma_{k} \geq \minprop(1 - \eta - k\gamma).
	\]
\end{lemma}
In the following discussion, we focus on balanced clusters, i.e., $\minprop \gtrsim 1/k$. We also assume that $\{\bbeta_j\}$ satisfy Condition \ref{cond:approx_orthonormal} with $\eta \lesssim 1$ and $\gamma \lesssim 1/k$, which leads to $\sigma_k = \Omega(\minprop)$ according to Lemma \ref{lem:lower_bound_sigma}. Now we provide two remarks for Theorem \ref{thm:tensor}.
\begin{remark}[Sample Complexity] We treat $\delta$ in Theorem \ref{thm:tensor} as a constant. Then \eqref{eq:n_1} implies that $n = n_1 + n_2 = \order(\varepsilon^{-2}k^6p\log k \log^3(p/\varepsilon))$ is sufficient to guarantee that the estimates produced by Algorithm \ref{alg:tensor} have accuracy at most $\varepsilon$. Moreover, we have $n_1 = \apporder(\varepsilon^{-2} k^6p)$, $n_2 = \apporder(\varepsilon^{-2}(k^6 + k^4p))$, which indicates that more samples are required to compute $\M_2$ than $\M_3$. To provide some intuitions why this conclusion makes sense, note that the estimation accuracy of $\overline{\M}_2$ determines the accuracy of identifying the subspace spanned by $\{\bbeta_j\}$ in the original $p$-dimensional space. While $\M_3$ has higher order, it is only required to concentrate well on a $k$-dimensional subspace computed from $\M_2$ thanks to the whitening procedure. It turns out subspace accuracy has a more critical impact on the final error and needs to sharpened with more samples.
\end{remark}

\begin{remark}[Time Complexity] \label{remark:tensor_time} Except the line 6 in Algorithm \ref{alg:tensor}, the other steps have total complexity $\order(n(p^2 + k^3))$. Note that it's not necessary to compute $\M_3$ directly since we can compute $\widetilde{\M}_3$ from whitened covariate vectors $\W^{\top}\x_i$. Running time of robust tensor power method is $\order(k^4NL)$. According to Lemma \ref{lem:robust_tensor}, it is sufficient to set $N = \order(\log k + \log\log(1/\varepsilon))$ and $L = \order(poly(k))$ for some polynomial function $poly(\cdot)$. When $k$ is large enough, $L$ can be very close to be linear in $k$ (see Theorem 5.1 in \cite{anandkumar2014tensor} for details). Roughly, we take $L = \order(k^2)$, which gives the running time of Algorithm \ref{alg:t_power} as $\order(k^6\log k)$ when $\varepsilon \gtrsim poly(1/k)$. Therefore, the overall complexity of Algorithm \ref{alg:tensor} is $\mathcal{O}(n(p^2 + k^3) + k^6\log k)$.	
\end{remark}

\subsection{Analysis of Alternating Minimization} \label{sec:altmin}
Now we turn to the analysis of Algorithm \ref{alg:altmin}. Let $\varepsilon_0 := \mathcal{E}(\{\bbeta_j^{(0)}\})$.
\begin{theorem}[Alternating Minimization] \label{thm:altm} Consider Algorithm \ref{alg:altmin} for successively refining estimation of $\{\bbeta_j\}$. Pick any $\delta \in (0,1)$. There exist constants $C_i$ such that the following holds. Suppose
\[
\varepsilon_0 \leq C_1\left(\frac{1}{k^2}\wedge \minprop\right)\Delta,~~ p \geq \log(2k^2T/\delta),
\]
and $n$ satisfies 
\begin{equation} \label{eq:n/T}
n/T \geq C_2\left(\frac{kp}{\minprop} \vee \frac{\log(8k^2T/\delta)}{\minprop^2}\right).
\end{equation}
With probability at least $1 - \delta$, $\{\bbeta_j^{(t)}\}$ satisfies
\[
\mathcal{E}(\{\bbeta_j^{(t)}\}) \leq \left(\frac{1}{2}\right)^t\cdot\varepsilon_0, ~~\text{for}~~t =1,\ldots,T. 
\]	 
\end{theorem}
See Section \ref{sec:proof:thm:altm} for the proof of the above result. Theorem \ref{alg:altmin} suggests that with good enough initialization, iterates $\{\bbeta_j^{(t)}\}$ have at least linear convergence to the ground truth parameters. Due to the fast convergence, it is sufficient to set $T = \order(\log(1/\varepsilon))$ to obtain estimation with accuracy $\varepsilon$. In the case of well balanced clusters, i.e. $\minprop \gtrsim 1/k$, $\varepsilon_0$ is required to be $\order(\Delta/k^2)$ in order to guarantee the convergence to global optima. Next, we give two remarks for sample and time complexities. In our discussion, we assume $\minprop \gtrsim 1/k$ and that $\delta$ is a small constant.
\begin{remark}[Sample Complexity]
	For accuracy $\varepsilon$, it is sufficient to have $n = \order(k^2p\log(1/\varepsilon))$ when $p$ satisfies $p \gtrsim \log k + \log\log(1/\varepsilon)$. Compared to the sample complexity of tensor decomposition, AltMin avoids the high-order polynomial factor of $k$. Moreover, it also changes the dependence on $\varepsilon$ from $1/\varepsilon^2$ to $\log(1/\varepsilon)$, which is a big save especially when we focus on exact recovery, which can happen as we show in the next section, after one step of AltMin when $\varepsilon \lesssim 1/p$. Notably, the statistical efficiency comes from a good initialization provided by tensor/spectral method. On one hand, AltMin alleviates the statistical inefficiency of spectral method; on the other hand, spectral method resolves the algorithmic intractability of AltMin.
\end{remark}
\begin{remark}[Time Complexity] \label{remark:time_altmin}
	Each iteration of AltMin has time complexity $\order(np^2/T + kp^3)$. Hence, the overall running time is $\order(np^2 + kp^3\log(1/\varepsilon))$\footnote{Factor $p^3$ in the second term stands for the complexity of inverting a $p$-by-$p$ matrix by Gauss-Jordan elimination. It can be further reduced by more complicated algorithms such as Strassen algorithm that has $\order(p^{2.807})$.}. Using the minimum requirement of $n$, we obtain complexity $\apporder(k^2p^3)$. Recall that solving linear regression by most practical algorithms has complexity $\mathcal{O}(p^3)$. Therefore, even labels are available, solving $k$ sets of linear equations requires time $\mathcal{O}(kp^3)$. AltMin almost has an extra factor $k$ as the price for addressing latent variables.
\end{remark}

\subsection{Exact Recovery and Overall Guarantee} \label{sec:overall}
We now consider putting the previous analysis of tensor decomposition and AltMin together to show exact recovery of $\{\bbeta_j\}$.
\begin{lemma}\label{lem:exact recovery}
	Pick any $\delta \in (0,1)$. For any fixed estimates $\{\widehat{\bbeta}_j\}_{j=1}^k$ and some constant $C$, if
	\[
	n \geq C \frac{1}{\minprop}\left( p \vee \log(k/\delta)\right) ~~\text{and}~~ \mathcal{E}(\{\widehat{\bbeta}_j\}) \leq \frac{\delta}{4nk}\Delta,
	\]
	Running one step of alternating minimization according to \eqref{eq:label_assignment} and \eqref{eq:parameter_update} using $n$ samples and initial guess $\{\widehat{\bbeta}_j\}$ produces true parameters $\{\bbeta_j\}$ with probability at least $1 - \delta$.
\end{lemma}

We provide the proof of the above result in Section \ref{sec:proof:lem:exact recovery}. Putting all ingredients together, we have the following overall guarantee:
\begin{corollary}[Exact Recovery] \label{cor:exact_recovery} Consider splitting $n$ samples from \eqref{eq:eq_xy} into two disjoint sets with size $\initsample, \altsample$ as inputs of Algorithm \ref{alg:tensor} and \ref{alg:altmin} for solving mixed linear equations as a two-stage method. Pick any $\delta \in (0,1)$. There exist constants $C_i$ such that the following holds. If we choose $T = C_1\log(k\altsample/\delta)$ in Algorithm \ref{alg:altmin}, and $(\initsample, \altsample, p)$ satisfy
\[
\initsample \geq C_2\left(\frac{(k^4 + 1/\minprop^2)(p/\sigma_k^2 + k^2 + p)\log(k/\delta)}{\minprop \sigma_k^3\Delta^2}	\log^3(\initsample) + \frac{k}{\minprop \delta}\right),
\]
\[
\altsample \geq C_3\left(\frac{kp}{\minprop} + \frac{p}{\minprop^2}\right)\log(kn_{\altsample}/\delta),
\] 
and
\[
p \geq C_4\left[\log\left(\frac{k}{\delta}\right) + \log \log\left(\frac{k\altsample}{\delta}\right)\right],
\]
then with probability at least $1 - \delta$, we have exact recovery, i.e. $\{\bbeta_j^{(T)}\}_{j=1}^k = \{\bbeta_j\}_{j = 1}^k$.

\end{corollary}
The proof is provided in Section \ref{sec:proof:cor:exact_recovery}. When $\minprop \gtrsim 1/k$ and Condition \ref{cond:approx_orthonormal} holds with $\gamma \lesssim 1$ and $\eta \lesssim 1/k$ ($\Delta \gtrsim 1$ in the case), Corollary \ref{cor:exact_recovery} implies that $n = \initsample + \altsample = \order(k^{10}p\log k \log^3p)$ is enough for exact recovery with high probability, say $99\%$. With this amount of samples, Remarks \ref{remark:tensor_time} and \ref{remark:time_altmin} give the overall time complexity as $ \order(k^{10}p(p^2 + k^3) \log k \log^3p)$. Note that solving $k$ sets of linear equations (labels are known) needs at least $kp$ samples, and usually requires time $\order(kp^3)$. Hence, under the aforementioned setting, our two-stage algorithm is nearly optimal in $p$ with respect to sample and time complexities.

\section{Numerical Results} \label{sec:numeric}
In this section, we provide some numerical results to demonstrate the empirical performance of the proposed method (combination of Algorithms \ref{alg:tensor} and \ref{alg:altmin}) for solving mixed linear equations, and also compare it with random initialized Alternating minimization (AltMin). All algorithms are implemented in MATLAB. While sample-splitting is useful for our theoretical analysis, we find it unnecessary in practice. Therefore, we remove the sample-splittings in Algorithms  \ref{alg:tensor} and \ref{alg:altmin}, and use the whole sample set in the entire process. AltMin is implemented to terminate when the label assignment no longer changes or the maximal number of iterations $T$ is reached. In all experiments, we set $T = 200$.

\paragraph{Datasets.} For given problem size $(n, p, k)$, we generate synthetic datasets as follows. Covariate vectors $\{\x_i\}_{i=1}^n$ are drawn independently from $\Gaussian(\bm{0}, \Identity_p)$. Model parameters $\{\bbeta_j\}_{j=1}^k$ are a random set of $k$ vectors in $\sphere^{p-1}$, where every two distinct $\bbeta$s have distance $\Delta = 1.2$. Therefore, these parameters are not orthogonal. Suppose $\B \in \real^{p\times k}$ denotes the matrix with $\bbeta_j$ as the $j$-th column. We let $\B = \Ub \bm{\Lambda}^{1/2} \Vb^{\top}$, where $\Ub \in \real^{p \times k}$ represents the basis of a random $k$-dimensional subspace in $\real^p$. Matrices $\bm{\Lambda}, \Vb \in \real^{k \times k}$ are from the eigen-decomposition of symmetric matrix $\bm{C} = \Vb \bm{\Lambda}\Vb^{\top}$, where the diagonal terms of $\bm{C}$ are $1$ and the rest entries are $1 - \Delta^2/2$. We assign equal weights $\omega_j = 1/k$ for all clusters.

\paragraph{Results.} Our first set of results, presented in Figure \ref{fig:trace}, show the convergence of estimation errors of AltMin with random and tensor initializations. Recall that estimation error is defined in \eqref{eq:error}. In random setting, AltMin starts with a set of uniformly random $k$ vectors in $\sphere^{p-1}$. We find that AltMin with random starting points has quite slow convergence, and fails to produce true $\bbeta$s with significant probability. In contrast, with the same amount of samples, tensor method provides more accurate starting points, which leads to much faster convergence of AltMin to the global optima. These results thus back up our convergence theory of AltMin (Theorem \ref{thm:altm}), and demonstrate the power of using tensor decomposition initialization.

The second set of results, presented in Figure \ref{fig:stats}, explore the statistical efficiency of the proposed algorithm---tensor initialized AltMin. For fixed $k = 3$, Figure \eqref{fig:stats_k3} reveals a linear dependence of the necessary sample size on $p$, which matches our results in Corollary \ref{cor:exact_recovery}. With fixed $p$, Figure \eqref{fig:stats_p10} indicates that $\order(k^3)$ samples could be enough in practice, which is much better than our theoretical guarantee $\order(k^{10})$. Sharpening the polynomial factor on $k$ is an interesting direction of future research.
\begin{figure*}[t] 
	\centering
	\begin{subfigure}[t]{0.32\textwidth}
		\centering
		\includegraphics[width=\textwidth]{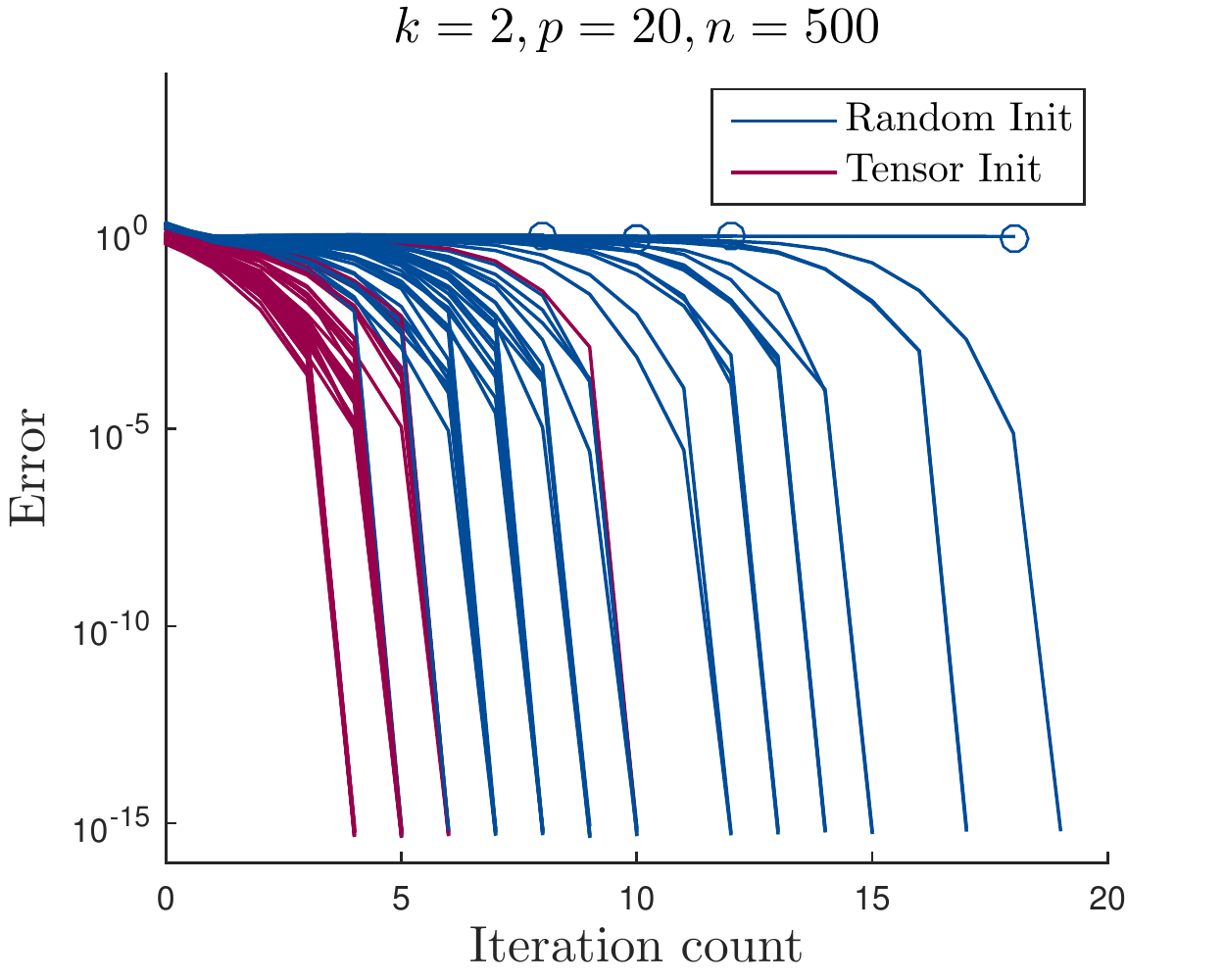}
		\caption{}
		\label{fig:trace_k2}
	\end{subfigure}%
	~ 
	\begin{subfigure}[t]{0.32\textwidth}
		\centering
		\includegraphics[width=\textwidth]{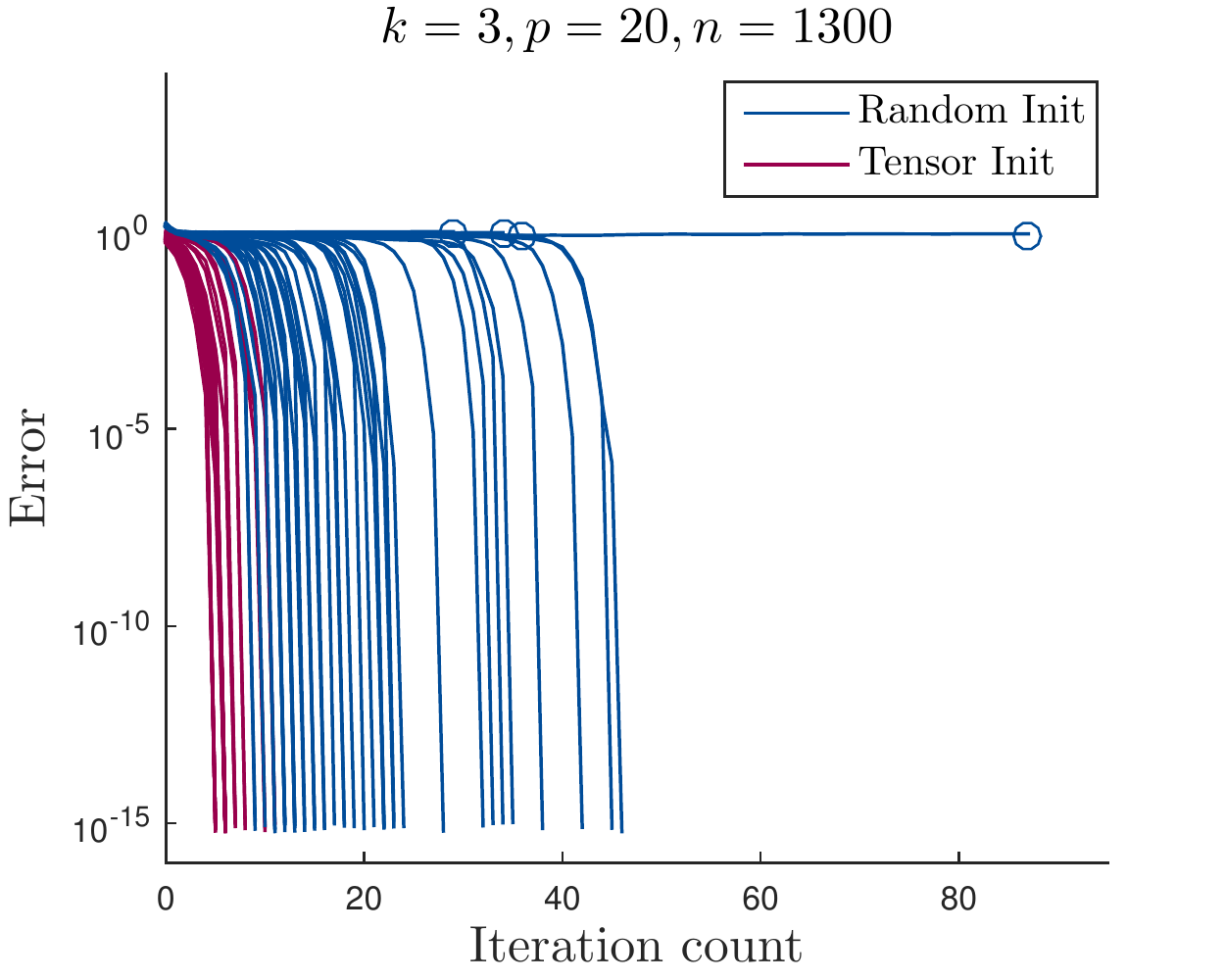}
		\caption{}
		\label{fig:trace_k3}
	\end{subfigure}
	~
	\begin{subfigure}[t]{0.32\textwidth}
		\centering
		\includegraphics[width=\textwidth]{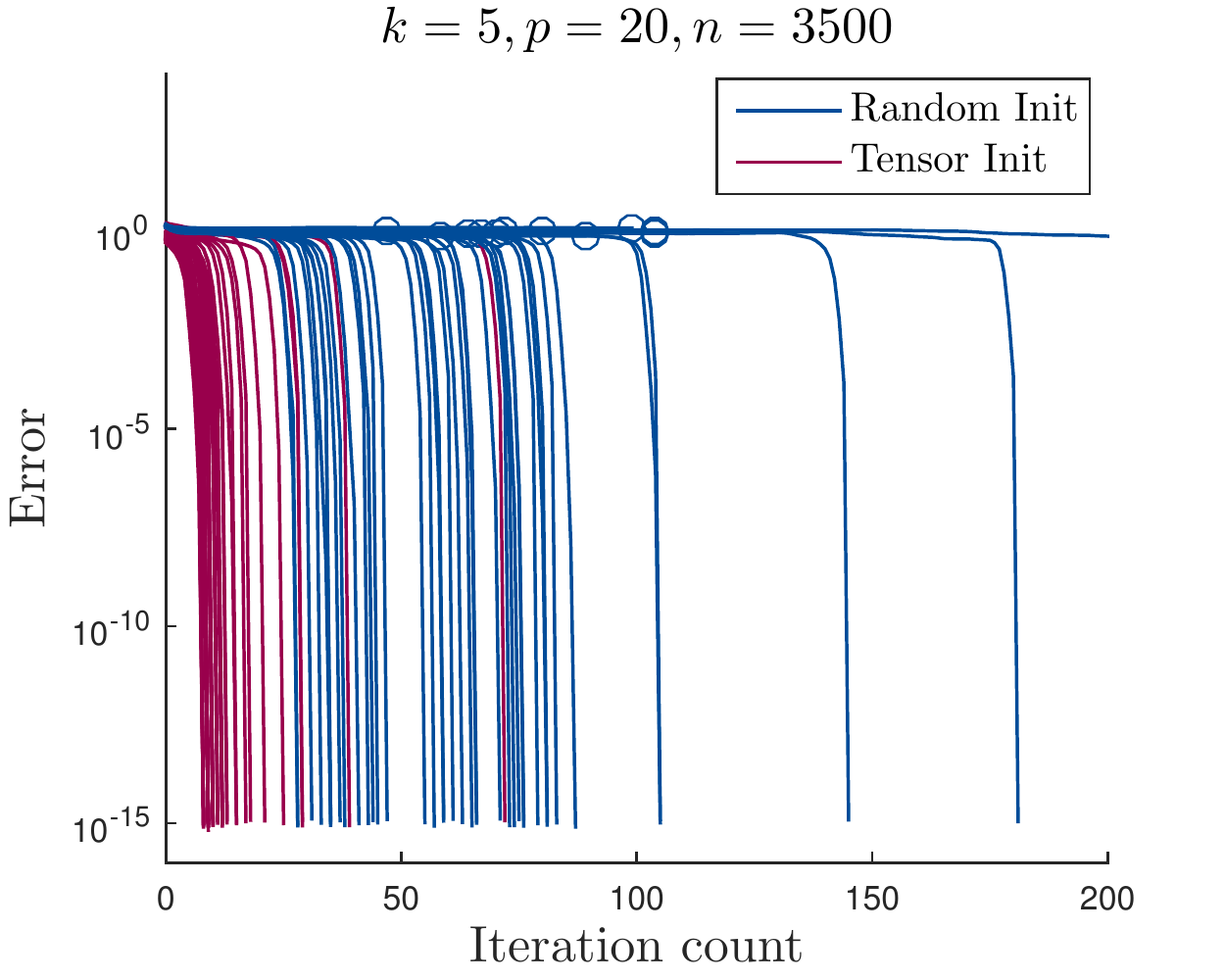}
		\caption{}
		\label{fig:trace_k5}
	\end{subfigure}
	\caption{Plot of estimation error (log scale) versus number of iterations in AltMin. Each panel shows $50$ trials for random and tensor initializations respectively. The circle markers indicate the terminations of AltMin due to local minima, i.e., the label assignments do not change in two consecutive iterations. Tensor decomposition is implemented with $L = 200k^2, N = 20\log(k)$.}
	\label{fig:trace}
\end{figure*}

\begin{figure*}[t] 
	\centering
	\begin{subfigure}[t]{0.49\textwidth}
		\centering
		\includegraphics[width=\textwidth]{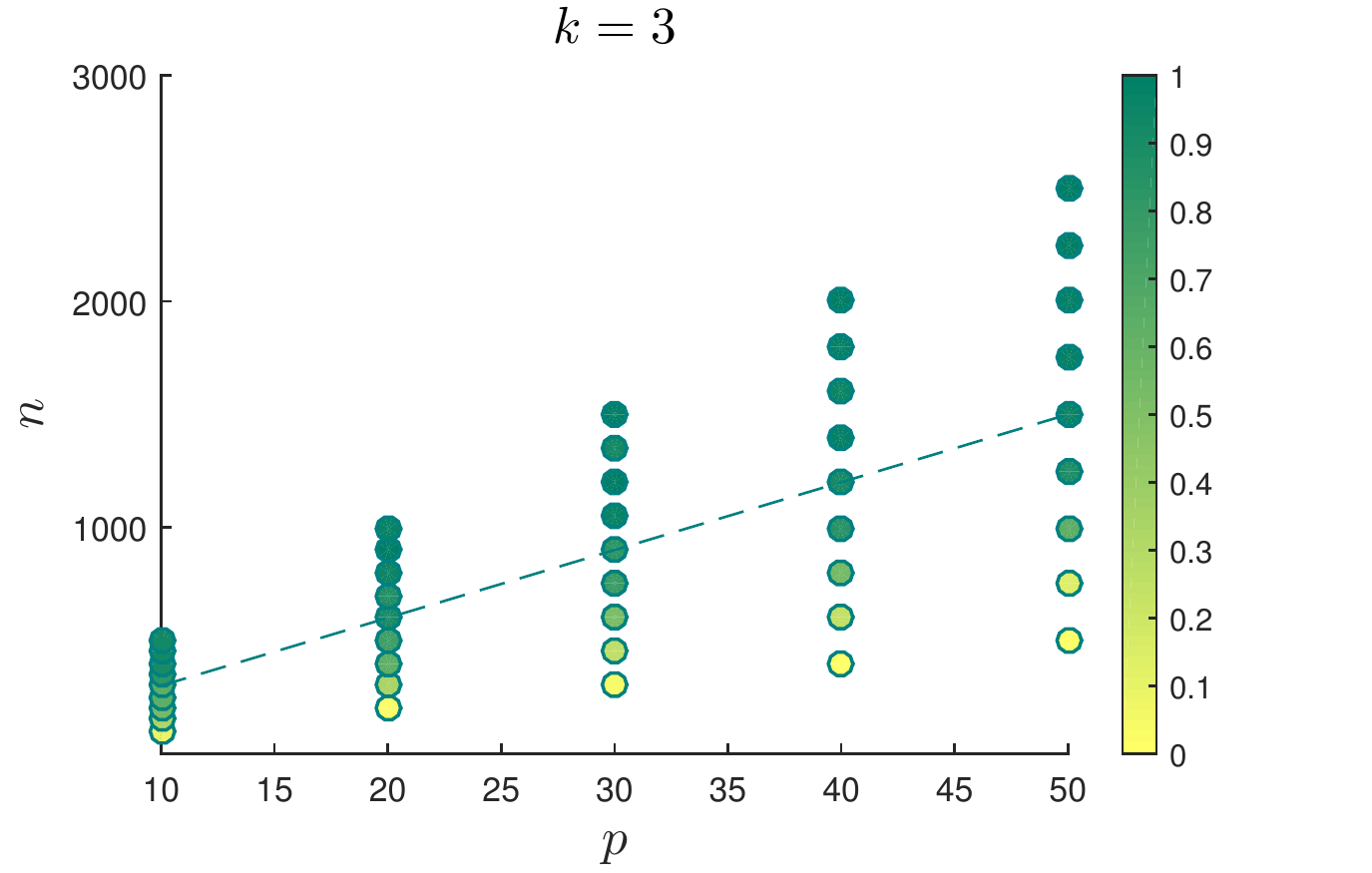}
		\caption{}
		\label{fig:stats_k3}
	\end{subfigure}%
	~ 
	\begin{subfigure}[t]{0.49\textwidth}
		\centering
		\includegraphics[width=\textwidth]{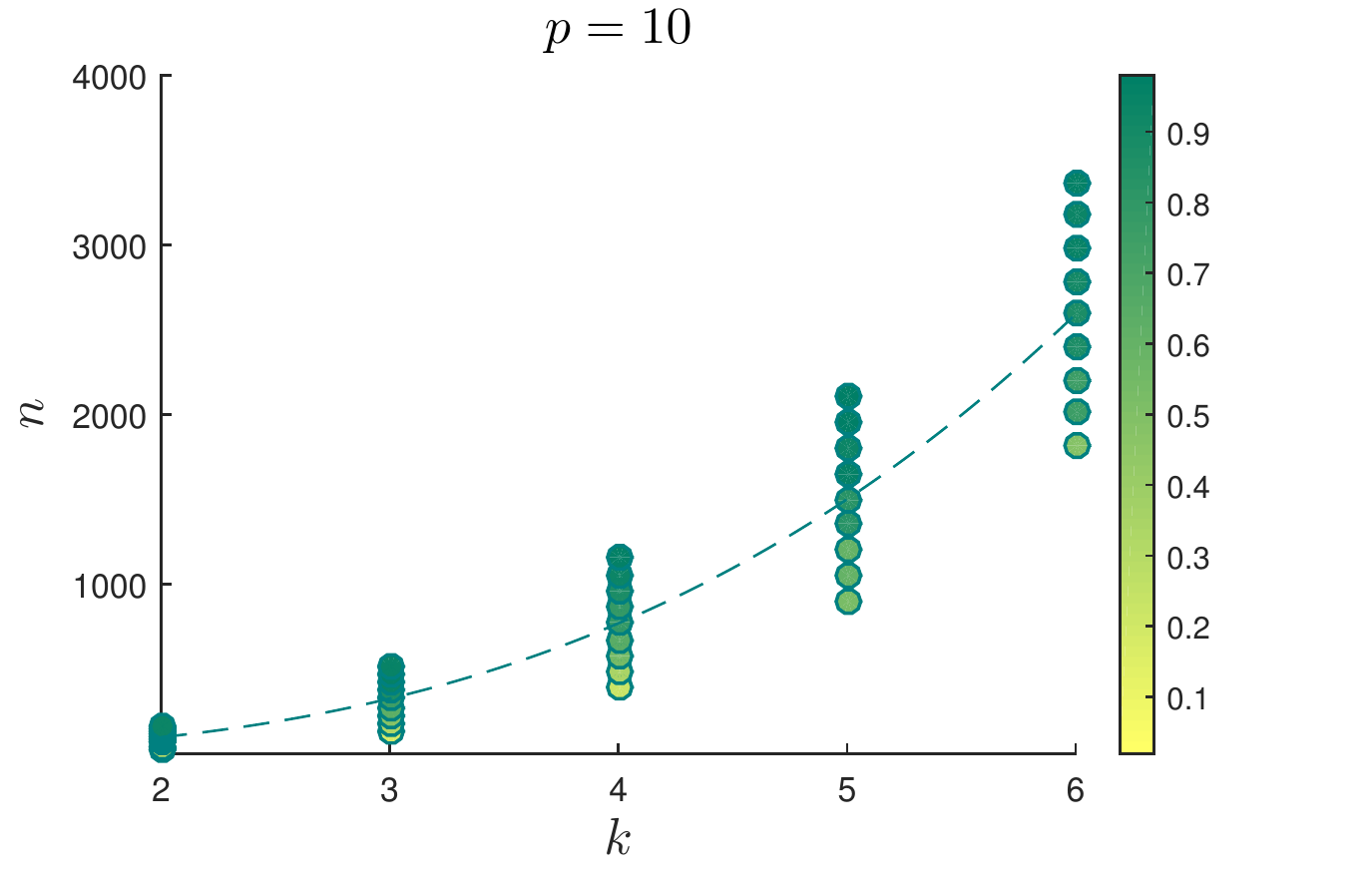}
		\caption{}
		\label{fig:stats_p10}
	\end{subfigure}
	\caption{Exact recovery probability of ``Tensor $+$ AltMin'' with varied $(n, p, k)$. The color of every dot represents the recovery probability computed from $100$ independent trials according to the colorbar on the right side. Tensor decomposition is implemented with $L = 200k^2, N = 20\log(k)$. The dashed line in $(a)$ shows function $n = 30p$. The dashed line in $(b)$ shows function $n = 12k^3$.} 
	\label{fig:stats}
\end{figure*}

\section{Proofs} \label{sec:proofs}
In this section, we provide proofs for Lemma \ref{lem:expectation} and the results presented in Section \ref{sec:theory}.
\subsection{Proof of Lemma \ref{lem:expectation}} \label{sec:proof:lem:expectation}
Recall that $z_i$ denotes the latent label associated with each sample. Suppose $X \sim \Gaussian(\bm{0},\Identity_p)$, and $Z$ has the distribution of each $z_i$. We find that
\begin{equation} \label{eq:exp_m0}
\Exs[m_0] = \sum_{j \in [k]}\Exs[\inprod{X}{\bbeta_j}^2]\cdot \Prob(Z = j) = \sum_{j \in [k]}\omega_j\twonorm{\bbeta_j}^2,
\end{equation}
\[
\Exs[\m_1] = \frac{1}{6}\sum_{j \in [k]}\Exs[\inprod{X}{\bbeta_j}^3X]\cdot \Prob(Z = j).
\]
One can check that for any $\bbeta$, $\Exs[\inprod{X}{\bbeta}^3X] = 3\twonorm{\bbeta}^2\bbeta$. Therefore,
\begin{equation} \label{eq:exp_m1}
\Exs[\m_1] = \frac{1}{2}\sum_{j \in [k]}\omega_j\twonorm{\bbeta_j}^2\bbeta_j.
\end{equation}
For $\M_2$, plugging \eqref{eq:exp_m0} into \eqref{eq:M_2_exp} yields
\[
\Exs[\M_2] = \frac{1}{2}\sum_{j \in [k]}\omega_j\Exs[\inprod{X}{\bbeta_j}^2X\otimes X] - \frac{1}{2}\sum_{j \in [k]}\omega_j\twonorm{\bbeta_j}^2\cdot\Identity_p.
\]
One can check $\Exs[\inprod{X}{\bbeta_j}^2X\otimes X] = 2\bbeta_j\bbeta_j^{\top} + \twonorm{\bbeta_j}^2$, which leads to $\Exs[\M_2] = \sum_{j \in [k]}\omega_j\bbeta_j\bbeta_j^{\top}$.

For $\M_3$, plugging \eqref{eq:exp_m1} into \eqref{eq:M_3_exp} gives
\[
\Exs[\M_3] = \frac{1}{6}\sum_{j \in [k]}\omega_j\Exs[\inprod{X}{\bbeta_j}^3X^{\otimes 3}] - \frac{1}{2}\sum_{j \in [k]}\omega_j\mathcal{T}(\twonorm{\bbeta_j}^2\bbeta_j).
\]
Then it remains to show that for any $\bbeta$,
\begin{equation} \label{eq:exp_m3}
\Exs[\inprod{X}{\bbeta}^3X^{\otimes 3}] = 6\bbeta^{\otimes 3} + 3\mathcal{T}(\twonorm{\bbeta}^2\bbeta).
\end{equation}
We directly verify the above inequality. Let $X = (X_1,\ldots,X_p)^{\top}$, $\bbeta = (\beta_1, \ldots, \beta_p)^{\top}$. For $(i,j,k) \in [p]\times [p]\times [p]$, let $L_{ijk}, R_{ijk}$ be the $(i,j,k)$-th entries of $\Exs[\inprod{X}{\bbeta}^3X^{\otimes 3}]$ and $6\bbeta^{\otimes 3} + 3\mathcal{T}(\twonorm{\bbeta}^2\bbeta)$ respectively. Due to symmetry, it is sufficient to consider the following cases.
\begin{itemize}
	\item $i \ne j \ne k \ne i$. We have $R_{ijk} = 6\beta_i\beta_j\beta_k$. Meanwhile,
	\[
	L_{ijk} = \Exs[\inprod{X}{\bbeta}^3X_iX_jX_k] = \Exs[6\beta_i\beta_j\beta_kX_i^2X_j^2X_k^2] = 6\beta_i\beta_j\beta_k.
	\]
	\item $i = j \ne k$. We have $R_{ijk} = 6\beta_i^2\beta_k + 3\twonorm{\bbeta}^2\beta_k$, and
	\begin{align*}
	L_{ijk} & = \Exs[\inprod{X}{\bbeta}^3X_i^2X_k] \\
	& = \Exs[\beta_k^3X_i^2X_k^4] + \Exs[3\beta_i^2\beta_kX_i^4X_k^2] + \sum_{t \in [p], t\ne i,k}\Exs[3\beta_k\beta_t^2X_i^2X_k^2X_t^2] \\
	& = 3\beta_k^3 + 9\beta_i^2\beta_k + 3\beta_k(\twonorm{\bbeta}^2 - \beta_i^2 - \beta_k^2) = 6\beta_i^2\beta_k + 3\twonorm{\bbeta}^2\beta_k.
	\end{align*}
	\item $i = j = k$. We have $R_{ijk} = 6\beta_i^3 + 9\twonorm{\bbeta}^2\beta_i$, and
	\begin{align*}
	L_{ijk} & = \Exs[\inprod{X}{\bbeta}^3X_i^3] = \Exs[\beta_i^3X_i^6] + \sum_{j \in [p], j \ne i}\Exs[3\beta_i\beta_j^2X_i^4X_j^2] \\
	& = 15\beta_i^3 + 9\beta_i(\twonorm{\bbeta}^2 - \beta_i^2) = 6\beta_i^2\beta_k + 3\twonorm{\bbeta}^2\beta_k.
	\end{align*}
\end{itemize}
In the above calculation, we frequently used the fact that odd-order moments of symmetric Gaussian is $0$. We finish proving \eqref{eq:exp_m3}, and thus conclude the proof.

\subsection{Proof of Lemma \ref{lem:lower_bound_sigma}} \label{sec:proof:lem:lower_bound_sigma}
Recall that $\sigma_k = \sigma_k(\overline{\M_2}) = \sigma_k(\sum_{j \in [k]}\omega_k\bbeta_k\bbeta_k^{\top})$. We always have
\[
\sigma_k \geq \minprop \sigma_k(\sum_{j \in [k]}\bbeta_k\bbeta_k^{\top}).
\]
Let $\B \in \real^{p \times k}$ be the matrix with columns $\bbeta_j$. Then we have
\[
\sigma_k(\sum_{j \in [k]}\bbeta_k\bbeta_k^{\top}) = \sigma_{min}(\B^{\top}\B).
\]
Thanks to the nearly orthonormal condition, matrix $\D = \B^{\top}\B$ has diagonal terms greater than $1 - \eta$ and the rest entries have magnitude smaller than $\gamma$. Therefore, for any $\ub \in \sphere^{k-1}$, we have
\[
\ub^{\top}\D\ub \geq (1 - \eta)\twonorm{\ub}^2 - \gamma\|\ub\|_{1}^2 \geq 1 - \eta - k\gamma,
\]
which completes the proof.

\subsection{Proof of Corollary \ref{cor:exact_recovery}} \label{sec:proof:cor:exact_recovery}
Linear convergence of AltMin requires $\varepsilon_0 \lesssim (1/k^2 \wedge \minprop)\Delta$. Plugging it as accuracy into Theorem \ref{thm:tensor} shows that it suffices to let
\[
\initsample \gtrsim \frac{(k^4 + 1/\minprop^2)(p/\sigma_k^2 + k^2 + p)\log(k/\delta)}{\minprop \sigma_k^3\Delta^2}	\log^3(\initsample) + \frac{k}{\minprop \delta}.
\]
The condition of $n$ in Lemma \ref{lem:exact recovery} is implied by the condition of $n/T$ in \eqref{eq:n/T}. Therefore, if $\{\bbeta_j^{(T-1)}\}_{j=1}^k$ produced by AltMin satisfies $\mathcal{E}(\{\bbeta_j^{(T-1)}\}) \leq \frac{\delta T}{4k\altsample}\Delta$, Lemma \ref{lem:exact recovery} implies that the $T$-th step of AltMin (using $\altsample/T$ samples) produces $\{\bbeta_j\}$ with high probability. Thanks to linear convergence, we have $\mathcal{E}(\{\bbeta_j^{(T-1)}\}) \leq (1/2)^{T-1}\Delta$. So it suffices to have	
\[
(1/2)^{T-1} \Delta \lesssim \frac{\delta T}{k\altsample}\Delta.
\]
Hence, choosing $T = C\log(k\altsample/\delta)$ with sufficiently large constant $C$ satisfies the above inequality. Plugging this choice of $T$ into \eqref{eq:n/T} shows that it suffices to let 
\[
\altsample \gtrsim \left(\frac{kp}{\minprop} + \frac{\log(k/\delta) + \log \log(k\altsample/\delta)}{\minprop^2}\right)\log(kn_{\altsample}/\delta).
\]
Condition on $p$ in Theorem \ref{thm:altm} then becomes $p \gtrsim \log(k/\delta) + \log \log(k\altsample/\delta)$, under which the above requirement of $\altsample$ can be strengthened to
\[
\altsample \gtrsim \left(\frac{kp}{\minprop} + \frac{p}{\minprop^2}\right)\log(kn_{\altsample}/\delta).
\]
	
\subsection{Proofs about Tensor Decomposition} \label{sec:proof:thm:tensor}
In this section, we prove the guarantee of tensor decomposition. Let $\overline{\M}_2 := \Exs[\M_2]$ and $\overline{\M}_3 := \Exs[\M_3]$. The proof idea of Theorem \ref{thm:tensor} is to show how approximate the empirical moments are to their expectations, and then establish the dependence between errors of approximation and estimation. Therefore, our proofs break down into the next two subsections. In Section \ref{sec:error_transfer}, given the approximation errors of moments
\begin{align}
\epsilon_2 & := \opnorm{\M_2 - \overline{\M}_2}, \label{eq:epsilon_2}\\
\epsilon_3 & := \opnorm{\M_3(\W,\W,\W) - \overline{\M}_3(\W,\W,\W)}, \label{eq:epsilon_3}
\end{align}
we follow the processes shown in Algorithm \ref{alg:tensor} to obtain an upper bound of the estimation error $\mathcal{E}(\{\bbeta_j^{(0)}\}) $ in terms of $\epsilon_2$ and $\epsilon_3$. In Section \ref{sec:concentration}, the dependence between $\epsilon_2, \epsilon_3$ and sample size is revealed by concentration analysis. We put these two parts together in Section \ref{sec:ensemble_proof_tensor} to prove Theorem \ref{thm:tensor}. 

\subsubsection{Error Transfer} \label{sec:error_transfer}
We now turn to show the how error is transfered from approximation bound to initial estimation. Recall that the robust tensor power method is run on tensor $\M_3(\W,\W,\W)$. We let $\overline{\W}$ be the whitening matrix of $\overline{\M}_2$. Then tensor  $\overline{\M}_3(\overline{\W},\overline{\W},\overline{\W})$ has orthogonal factorization
\[
\overline{\M}_3(\overline{\W},\overline{\W},\overline{\W}) = \sum_{j=1}^{k}\omega_j'\bbeta_j'\otimes\bbeta_j'\otimes\bbeta_j',
\] 
where $\omega_j' = 1/\sqrt{\omega_j}$, $\bbeta_j' = \sqrt{\omega_j}\overline{\W}^{\top}\bbeta_j$ and $\twonorm{\bbeta_j'} = 1$ for all $j \in [k]$.
We will use the next theory of robust tensor power method presented in \cite{anandkumar2014tensor}.
\begin{lemma}[Guarantee of Robust Tensor Power Method, Theorem 5.1 in \cite{anandkumar2014tensor}] \label{lem:robust_tensor} Suppose  $\T \in \real^{k\times k\times k}$ is a tensor with decomposition $\T = \sum_{j=1}^k \lambda_j \bbeta_j^{\otimes 3}$
where every $\lambda_j > 0$ and $\{\bbeta_j\}$ are orthonormal. Put $\overline{\lambda} := \max_{j \in [k]}\{\lambda_j\}, \underline{\lambda} := \min_{j \in [k]}\{\lambda_j\}$. Let $\widehat{\T} = \T + \E$ be the input of Algorithm \ref{alg:t_power}, where $\E$ is a symmetric tensor with $\opnorm{\E} \leq \epsilon$. There exist constants $C_i$ such that the following holds. Suppose $\epsilon \leq  C_1\underline{\lambda}/k$. For any $\delta \in (0,1)$, suppose $(N,L)$ in Algorithm \ref{alg:t_power} satisfy
\[
N \geq C_2\cdot\left(\log k  + \log \log (\overline{\lambda}/\epsilon)\right), ~~ L \geq C_3\cdot \text{poly}(k)\log(1/\delta),
\]
for some polynomial function $\text{poly}(\cdot)$. With probability at least $1 - \delta$, $\{(\widehat{\lambda}_j, \estbeta_j)\}$ returned by Algorithm \ref{alg:t_power} satisfy the bound
\[
\twonorm{\estbeta_{j} - \bbeta_{\pi(j)}} \leq 8\epsilon/\lambda_{\pi(j)},~~ \abs{\widehat{\lambda}_{j} - \lambda_{\pi(j)}} \leq 5\epsilon, ~~\text{for all}~~ j \in [k],
\]
where $\pi(\cdot)$ is some permutation function on $[k]$.
\end{lemma}

Without loss of generality, we set the permutation $\pi(\cdot)$ in the above result to be identity. Lemma \ref{lem:robust_tensor} implies that if 
\[
\epsilon := \opnorm{\overline{\M}_3(\overline{\W},\overline{\W},\overline{\W}) - \M_3(\W,\W,\W)} \lesssim \frac{1}{k},
\]
then with high probability, $\{(\widetilde{\omega}_j, \widetilde{\bbeta}_j)\}_{j=1}^k$ produced in the line 6 of Algorithm \ref{alg:tensor} satisfy
\[
\twonorm{\widetilde{\bbeta}_j - \bbeta_j'} \leq 8\epsilon/\omega_j' = 8\epsilon\sqrt{\omega_j},~~ \abs{\widetilde{\omega}_j - \omega_j'} \leq 5\epsilon.
\]
Then we have
\begin{align*}
\twonorm{\bbeta_j^{(0)} - \bbeta_j} & = \twonorm{\widetilde{\omega}_j(\W^{\top})^{\dagger}\widetilde{\bbeta}_j - \omega_j'(\overline{\W}^{\top})^{\dagger}\bbeta_j'} \\
& \leq \twonorm{\widetilde{\omega}_j(\W^{\top})^{\dagger}\widetilde{\bbeta}_j - \widetilde{\omega}_j(\overline{\W}^{\top})^{\dagger}\bbeta_j'} +\twonorm{ \widetilde{\omega}_j(\overline{\W}^{\top})^{\dagger}\bbeta_j' - \omega_j'(\overline{\W}^{\top})^{\dagger}\bbeta_j'} \\
& \leq \twonorm{\widetilde{\omega}_j(\W^{\top})^{\dagger}\widetilde{\bbeta}_j - \widetilde{\omega}_j(\overline{\W}^{\top})^{\dagger}\bbeta_j'} + 5\epsilon \opnorm{\overline{\W}^{\dagger}} \\
& \leq \twonorm{\widetilde{\omega}_j(\W^{\top})^{\dagger}\widetilde{\bbeta}_j - \widetilde{\omega}_j(\W^{\top})^{\dagger}\bbeta_j'} + \twonorm{\widetilde{\omega}_j(\W^{\top})^{\dagger}\bbeta_j' - \widetilde{\omega}_j(\overline{\W}^{\top})^{\dagger}\bbeta_j'} + 5\epsilon \opnorm{\overline{\W}^{\dagger}} \\
& \leq \twonorm{\widetilde{\omega}_j(\W^{\top})^{\dagger}\widetilde{\bbeta}_j - \widetilde{\omega}_j(\W^{\top})^{\dagger}\bbeta_j'} + \opnorm{\W^{\dagger} - \overline{\W}^{\dagger}} + 5\epsilon \opnorm{\overline{\W}^{\dagger}} \\
& \leq \opnorm{\W^{\dagger}}\cdot\twonorm{\widetilde{\bbeta}_j - \bbeta_j'} + \opnorm{\W^{\dagger} - \overline{\W}^{\dagger}} + 5\epsilon \opnorm{\overline{\W}^{\dagger}} \\
& \leq  8\epsilon  \sqrt{\omega_j}  \opnorm{\W^{\dagger}} + \opnorm{\W^{\dagger} - \overline{\W}^{\dagger}} + 5\epsilon \opnorm{\overline{\W}^{\dagger}}.
\end{align*}

Recall that $\sigma_k = \sigma_k(\overline{\M}_2)$. Put $\sigma_1 = \sigma_1(\overline{\M}_2)$. Let $\M_2'$ be the best rank $k$ approximation of $\M_2$. We have 
\[
\opnorm{\M_2' - \overline{\M}_2} \leq \epsilon_2 + \sigma_{k+1}(\M_2) \leq 2\epsilon_2,
\]
where the last step follows from Weyl's theorem. Using the properties of whitening in Lemma \ref{lem:whitening} by replacing $\A, \widehat{\A}$ with $\overline{\M}_2, \M_2'$, when $\epsilon_2/\sigma_k \leq 1/6$, we have
\[
 \opnorm{\W^{\dagger}} \leq 2\opnorm{\overline{\W}^{\dagger}} = 2\sqrt{\sigma_1}, 
\]
\[
\opnorm{\W^{\dagger} - \overline{\W}^{\dagger}} \leq 4\epsilon_2\opnorm{\overline{\W}^{\dagger}}/\sigma_k = 4\epsilon_2\sqrt{\sigma_1}/\sigma_k.
\]
We thus obtain
\begin{equation} \label{eq:err_tmp}
\twonorm{\bbeta_j^{(0)} - \bbeta_j} \leq 21\sqrt{\sigma_1}\epsilon+ \frac{4\sqrt{\sigma_1}\epsilon_2}{\sigma_k}.
\end{equation}

It remains to relate $\epsilon$ to $\epsilon_2$ and $\epsilon_3$. We apply a series of triangle inequalities as follows.
\begin{align}
\epsilon & = \opnorm{\overline{\M}_3(\overline{\W},\overline{\W},\overline{\W}) - \M_3(\W,\W,\W)} \notag\\
& = \opnorm{\overline{\M}_3(\overline{\W},\overline{\W},\overline{\W}) - \overline{\M}_3(\overline{\W},\overline{\W},\W) + \overline{\M}_3(\overline{\W},\overline{\W},\W) - \M_3(\W,\W,\W)} \notag\\
& \leq \opnorm{\overline{\M}_3(\overline{\W},\overline{\W},\W) - \overline{\M}_3(\overline{\W},\W,\W) + \overline{\M}_3(\overline{\W},\W,\W) - \M_3(\W,\W,\W)} \notag \\
& ~~~~ + \opnorm{\overline{\M}_3(\overline{\W},\overline{\W},\overline{\W} - \W)} \notag\\
& \leq \opnorm{\overline{\M}_3} \opnorm{\overline{\W}}^2\opnorm{\overline{\W} - \W} + \opnorm{\overline{\M}_3(\overline{\W},\overline{\W},\W) - \overline{\M}_3(\overline{\W},\W,\W)} \notag\\
& ~~~~ + \opnorm{\overline{\M}_3(\overline{\W},\W,\W) - \M_3(\W,\W,\W)} \notag\\
& \leq \opnorm{\overline{\M}_3} \opnorm{\overline{\W}}^2\opnorm{\overline{\W} - \W} + \opnorm{\overline{\M}_3} \opnorm{\overline{\W}}\opnorm{\W}\opnorm{\overline{\W} - \W} \notag\\
& ~~~~ + \opnorm{\overline{\M}_3} \opnorm{\W}^2\opnorm{\overline{\W} - \W} + \epsilon_3. \label{eq:M3_tri}
\end{align}
Applying Lemma \ref{lem:whitening} again, we have
\begin{equation}\label{eq:op_whiten}
\opnorm{\W} \leq 2\opnorm{\overline{\W}} = 2/\sqrt{\sigma_k}, ~~\opnorm{\overline{\W} - \W} \leq 4\epsilon_2/\sqrt{\sigma_k}^3. 
\end{equation}
Plugging it back into the last line of \eqref{eq:M3_tri} yields
\begin{equation} \label{eq:M_3a}
\opnorm{\overline{\M}_3(\overline{\W},\overline{\W},\overline{\W}) - \M_3(\W,\W,\W)} \leq \frac{28\epsilon_2}{\sqrt{\sigma_k}^5}\opnorm{\overline{\M}_3} + \epsilon_3.
\end{equation}
We thus obtain the following error bound by putting \eqref{eq:err_tmp} and \eqref{eq:M_3a} together:
\begin{equation} \label{eq:error_transfer}
\twonorm{\bbeta_j^{(0)} - \bbeta_j} \lesssim \frac{\sqrt{\sigma_1}\epsilon_2}{\sigma_k} + \sqrt{\sigma_1}\epsilon_3 + \frac{\sqrt{\sigma_1}\opnorm{\overline{\M}_3}\epsilon_2}{\sqrt{\sigma_k}^5}, ~\text{for all}~ j \in [k].
\end{equation}

Recall that, in order to obtain \eqref{eq:err_tmp}, we have to make sure $\epsilon \lesssim 1/k$ as required in Lemma \ref{lem:robust_tensor}. Then inequality \eqref{eq:M_3a} indicates that it's sufficient to require 
\begin{equation} \label{eq:require_ep_23}
\epsilon_2 \lesssim \frac{\sqrt{\sigma_k}^5}{k\opnorm{\overline{\M}_3}}~~\text{and}~~\epsilon_3 \lesssim \frac{1}{k},
\end{equation} 
which will be used in the concentration analysis.

\subsubsection{Concentration Analysis} \label{sec:concentration}
Now we turn to the analysis of the concentration of empirical moments, and we derive upper bounds on $\epsilon_2$ and $\epsilon_3$. Note that $\M_3$ involves Gaussian's high-order moments (up to $6$th moment). In order to deal with the heavy tail, we will leverage a {\em truncation argument}, where we introduce truncated response $y_i'$ as 
\begin{equation} \label{eq:truncation}
y_i' = \begin{cases} y_i, & ~\text{if}~ |y_i| \leq M, \\ \sign(y_i)\cdot M, & \text{otherwise} \end{cases},
\end{equation}
where $M > 0$ is some threshold chosen in our analysis. When $M$ is sufficiently large, we have $y_i = y_i'$ for all $i \in [n]$ with high probability, which means the tail bounds about $\{(y_i',\x_i)\}$ still apply to original samples $\{(y_i,\x_i)\}$. The advance of analyzing concentration using $(y_i',\x_i)$ is that $y_i'\cdot\x_i$ is sub-Gaussian random vector thanks to the boundedness of $y_i'$. One should note that truncating $y_i$ might change the expectation of moments slightly. Therefore, a tedious but important part of our analysis is to show that the expectation deviation from truncation is much smaller compared to the desired tail bound. In detail, we have the next result proved using the truncation idea. See Section \ref{sec:proof:lem:concentration} for the complete proof.
\begin{lemma}[Concentration of Empirical Moments of Single Model] \label{lem:concentration}
	Suppose $n$ samples $\x_1, \ldots, \x_n$ are generated from $\Gaussian(\bm{0}, \Identity_p)$ and $y_i = \inprod{\x_i}{\bbeta}$ for some fixed $\bbeta \in \sphere^{p-1}$. Let
	\[
	\m_1 = \frac{1}{n}\sum_{i \in [n]}y_i^3\x_i,~ \M_2 = \frac{1}{n}\sum_{i \in [n]}y_i^2\x_i^{\otimes 2},~ \M_3 = \frac{1}{n}\sum_{i \in [n]}y_i^3\x_i^{\otimes 3}.
	\]
	Moreover, let $\overline{m}_0 = \Exs[m_0], \overline{\m}_1 = \Exs[\m_1], \overline{\M}_2 = \Exs[\M_2], \overline{\M}_3 = \Exs[\M_3]$. There exist constants $C_i$ such that the following holds. Pick any $\delta \in (0,1)$ and any fixed matrix $\Sb \in \real^{p\times s}$ with $s \leq p$.
	\begin{enumerate}
		\item If $n \geq C_1/\delta$, with probability at least $1 - \delta$, we have
		\begin{equation}\label{eq:concentration_1}
		\twonorm{\Sb^{\top}(\m_1 - \overline{\m}_1)} \leq C_2\opnorm{\Sb}\frac{\log^{3/2} n}{\sqrt{n}}\max\left\{\sqrt{\log\left(\frac{2}{\delta}\right)}, \sqrt{s}\right\}.
		\end{equation}
		
		\item If $n \geq C_3 \max\{1/\delta, s\}$, with probability $1 - \delta$, we have
		\begin{equation} \label{eq:concentration_2}
		\opnorm{\Sb^{\top}\left(\M_2 - \overline{\M}_2\right)\Sb} ~\leq~  C_4 \opnorm{\Sb}^2\frac{\log n}{\sqrt{n}}\max\left\{ \sqrt{\log\left(\frac{2}{\delta}\right)}, \sqrt{s}\right\}.
		\end{equation}
		
		\item If $n \geq C_5 \max\{ s\log\left(\frac{2}{\delta}\right), 1/\delta\}$, with probability at least $1 - \delta$, we have
		\begin{equation}\label{eq:concentration_3}
		\opnorm{\left(\M_3 - \overline{\M}_3\right)(\Sb,\Sb,\Sb)} ~\leq~ C_6 \opnorm{\Sb}^3\frac{s\log ^{3/2}n}{\sqrt{n}}\sqrt{\log\left(\frac{2}{\delta}\right)}.
		\end{equation}
	\end{enumerate}
\end{lemma}
This result provides concentration bounds of the moments constructed from single linear model. In the case of mixture samples $\{(y_i, \x_i)\}_{i=1}^{n}$, we can split the set into $k$ sets $\{(y_i^{(j)}, \x_i^{(j)})\}_{i=1}^{n_j}, j = 1,\ldots,k$, where the $j$-th set corresponds to linear model $\bbeta_j$. Therefore, for the moments given in \eqref{eq:m_01}-\eqref{eq:m_3}, we have
\begin{align*}
m_0 & = \sum_{j \in [k]} \widehat{\omega}_j m_0^{(j)},~~ \m_1 = \frac{1}{6}\sum_{j \in [k]} \widehat{\omega}_j \m_1^{(j)}, \\
\M_2 & = \frac{1}{2}\sum_{j \in [k]} \widehat{\omega}_j \M_2^{(j)} - \frac{1}{2}m_0\cdot \Identity_p,~~ \M_3 = \frac{1}{6}\sum_{j \in [k]} \widehat{\omega}_j \M_3^{(j)} - \mathcal{T}(\m_1),
\end{align*}
where $\widehat{\omega}_j$ denotes the empirical proportion of each model, and we let $m_0^{(j)} := \frac{1}{n_j}\sum_{i \in [n_j]}y_i^{(j)2}$, $\m_1^{(j)} := \frac{1}{n_j}\sum_{i \in [n_j]}y_i^{(j)2}\x_i^{(j)}$, $\M_2^{(j)} := \frac{1}{n_j}\sum_{i \in [n_j]}y_i^{(j)2}\x_i^{(j)\otimes 2}$, $\M_3^{(j)} := \frac{1}{n_j}\sum_{j \in [n_j]}y_i^{(j)3}\x_i^{(j)\otimes 3}$.

Next, we will derive concentration bounds for $m_0, \m_1, \M_2, \M_3$ respectively. To ease notation, for every moment, we use $n$ to denote the number of samples for computing it, while they might be computed from different sets of samples in Algorithm \ref{alg:tensor}.
\vskip .1in 
\paragraph{Bound of $|m_0 - \overline{m}_0|$.} We find that
\begin{align}
\epsilon_0 := |m_0 - \overline{m}_0| & \leq \sum_{j \in [k]}\widehat{\omega}_j\left|m_0^{(j)} - \Exs[m_0^{(j)}]\right| + \sum_{j \in [k]}|\widehat{\omega}_j - \omega_j|\cdot\Exs[m_0^{(j)}] \notag\\
& \leq \sup_{j \in [k]} \abs{m_0^{(j)} - \Exs[m_0^{(j)}]} + \sum_{j \in [k]}|\widehat{\omega}_j - \omega_j|\cdot\Exs[m_0^{(j)}] \notag\\
& \leq \sup_{j \in [k]} \abs{m_0^{(j)} - \Exs[m_0^{(j)}]} + \underbrace{\sum_{j \in [k]}|\widehat{\omega}_j - \omega_j|}_{\epsilon_{\omega}}, \label{eq:moment_decomp}
\end{align}
where the last step follows from the fact $\Exs[m_0^{(j)}] \leq 1$ due to the assumption $\max_{j \in [k]}\twonorm{\bbeta_j} = 1$. We first bound $\epsilon_{\omega}$. Note that $n\widehat{\omega}_j$ is a sum of $n$ Bernoulli random variables with success probability $\omega_j$. Lemma \ref{lem:sum_bern} gives that for any $t \in (0,1)$
\[
\Prob\left(|\widehat{\omega}_j - \omega_j| \geq t \omega_j\right) \leq 2e^{-\frac{3t^2}{2(t+3)}n\omega_j} \leq 2e^{-3t^2n\minprop/8}.
\]
Using union bound and setting $t = \sqrt{8\log(2k/\delta)/(3\minprop n)}$, which can be less than $1$ when $n \geq C\log(k/\delta)/\minprop$ for sufficiently large $C$, we obtain
\begin{equation} \label{eq:weight_approx}
\Prob\left(\epsilon_{\omega} \geq \sqrt{\frac{8\log(2k/\delta)}{3\minprop n}}\right) \leq 2ke^{-3t^2n\minprop/8} = \delta.
\end{equation}
Now we turn to the first term in \eqref{eq:moment_decomp}. Note that $y_i^2 = \inprod{\x_i}{\bbeta_j}^2$ is sub-Gaussian with constant Orlicz norm as $\twonorm{\bbeta_j} \leq 1$. Then by standard concentration of sub-Gaussian (e.g., \eqref{eq:concentration_scalar_2} with $p=1$), we find that there exist constants $C, C'$ such that if $n \geq C\frac{1}{\minprop}\log(k/\delta)$, we have
\[
\Prob\left(\sup_{j \in [k]} \abs{m_0^{(j)} - \Exs[m_0^{(j)}]} \geq C'\sqrt{\frac{1}{\minprop n}\log\left(\frac{k}{\delta}\right)} \right) \leq \delta.
\]
for any $\delta \in (0,1)$. Excluding the probability $\delta$, we obtain 
\begin{equation} \label{eq:epsilon_0_bound}
\epsilon_0 \lesssim \sqrt{\log(k/\delta)/(\minprop n)} + \epsilon_{\omega}.
\end{equation}
\vskip .1in
\paragraph{Bound of $\twonorm{\m_1 - \overline{\m}_1}$.} Similar to \eqref{eq:moment_decomp}, we have
\begin{align*}
\epsilon_1 := \twonorm{\m_1 - \overline{\m}_1} & \lesssim \sup_{j \in [k]} \twonorm{\m_1^{(j)} - \Exs[\m_1^{(j)}]} + \epsilon_{\omega}\cdot \sup_{j \in [k]} \twonorm{\Exs[\m_1^{(j)}]} \\
& \lesssim  \sup_{j \in [k]} \twonorm{\m_1^{(j)} - \Exs[\m_1^{(j)}]} + \epsilon_{\omega}.
\end{align*}
Using \eqref{eq:concentration_1} in Lemma \ref{lem:concentration} by setting $\Sb = \Identity_p$ and replacing $\delta$ with $\delta/k$, we have that the condition $n \gtrsim k/(\minprop\delta)$ leads to 
\[
\sup_{j \in [k]}  \twonorm{\m_1^{(j)} - \Exs[\m_1^{(j)}]} \lesssim \frac{\log^{3/2}(\minprop n)}{\sqrt{\minprop n}}\sqrt{p \log(2k/\delta)}
\]
holds with probability at least $1 - \delta$. Conditioning on this event leads to
\begin{equation} \label{eq:epsilon_1_bound}
\epsilon_1 \lesssim \log^{3/2}(\minprop n)\sqrt{p \log(2k/\delta)/(\minprop n)} + \epsilon_{\omega}.
\end{equation}

\vskip .1in
\paragraph{Bound of $\opnorm{\M_2 - \overline{\M}_2}$.} We find that
\begin{align*}
\epsilon_2 & \lesssim \opnorm{\sum_{j \in [k]}\widehat{\omega}_j \M_2^{(j)} - \sum_{j \in [k]}\omega_j \Exs[\M_2^{(j)}]} + |m_0 - \overline{m}_0| \\
& \lesssim \sup_{j \in [k]}\opnorm{\M_2^{(j)} - \Exs[\M_2^{(j)}]} + \epsilon_{\omega} + \epsilon_0,
\end{align*}
where the second step follows from similar calculation in \eqref{eq:moment_decomp} and the fact $\opnorm{\Exs[\M_2^{(j)}]} \lesssim 1$ for all $j \in [k]$. Applying  \eqref{eq:concentration_2} by choosing $\Sb = \Identity_p$ and setting $\delta$ to be $\delta/k$, we have that when $n \gtrsim \minprop^{-1}\max\{k/\delta, ~p\}$,
\[
\sup_{j \in [k]} \opnorm{\M_2^{(j)} - \Exs[\M_2^{(j)}]} \lesssim \log(\minprop n) \sqrt{p\log\left(2k/\delta\right)/(\minprop n)}
\]
holds with probability at least $1 - \delta$. Conditioning on the event, we conclude that 
\begin{equation} \label{eq:epsilon_2_bound}
\epsilon_2 \lesssim \log(\minprop n) \sqrt{p\log\left(2k/\delta\right)/(\minprop n)} + \epsilon_{\omega} + \epsilon_{0}.
\end{equation}

\vskip .1in
\paragraph{Bound of $\opnorm{\M_3(\W,\W,\W) - \overline{\M}_3(\W,\W,\W)}$.} Now we condition on the event $\epsilon_2/\sigma_k < 1/6$, which can lead to $\opnorm{\W} \leq 2/\sqrt{\sigma_k}$ as shown in \eqref{eq:op_whiten}. Let $\epsilon_{\mathcal{T}} := \opnorm{\mathcal{T}(\m_1 - \overline{\m}_1)(\W,\W,\W)}$. Recall that $\epsilon_3$ is defined in \eqref{eq:epsilon_3}. We find that
\begin{align*}
\epsilon_3 & \lesssim \opnorm{\sum_{j \in [k]}\widehat{\omega}_j\M_3^{(j)}(\W,\W,\W) - \sum_{j \in[k]}\omega_j\Exs[\M_3^{(j)}(\W,\W,\W)]} + \epsilon_{\mathcal{T}} \\
& \lesssim \sup_{j \in [k]}\opnorm{\M_3^{(j)}(\W,\W,\W) - \Exs[\M_3^{(j)}(\W,\W,\W)]} + \epsilon_{\omega} + \epsilon_{\mathcal{T}}.
\end{align*}
Again, the last step follows from similar steps in \eqref{eq:moment_decomp} and the fact that
\[
\opnorm{\Exs[\M_3^{(j)}(\W,\W,\W)]} \lesssim \opnorm{\Exs[\M_3^{(j)}]}\cdot\opnorm{\W}^3 \lesssim \opnorm{\W}^3 \lesssim 1/\sqrt{\sigma_k}^3.
\]
Note that $\W$ is computed from $\M_2$. Due to the sample splitting in Algorithm \ref{alg:tensor}, $\W$ is independent of $\M_3$. Therefore, we can apply \eqref{eq:concentration_3} by replacing $\Sb, \delta$ with $\W, \delta/k$ to obtain that
\begin{equation} \label{eq:M3_tmp}
\sup_{j \in [k]}\opnorm{\M_3^{(j)}(\W,\W,\W) - \Exs[\M_3^{(j)}(\W,\W,\W)]} \lesssim \frac{1}{\sqrt{\sigma_k}^3}k\log^{3/2}(\minprop n) \sqrt{\log(2k/\delta)/(\minprop n)}
\end{equation}
holds with probability at least $1 - \delta$ under condition $n \gtrsim k/(\minprop \delta)$. For $\mathcal{T}(\cdot)$, we have that for any $\ub \in \real^{p}$,
\begin{equation} \label{eq:T_op}
\opnorm{\mathcal{T}(\ub)} \leq 3\twonorm{\ub},
\end{equation}
which is proved at the end of this section. We have
\[
\epsilon_{\mathcal{T}} \lesssim \opnorm{\W}^3\epsilon_1 \lesssim \epsilon_1/\sqrt{\sigma_k}^3.
\]
Conditioning on \eqref{eq:M3_tmp} leads to
\begin{equation} \label{eq:epsilon_3_bound}
\epsilon_3 \lesssim \frac{1}{\sqrt{\sigma_k}^3}k\log^{3/2}(\minprop n) \sqrt{\log(2k/\delta)/(\minprop n)} + \epsilon_{\omega} + \epsilon_1/\sqrt{\sigma_k}^3.
\end{equation}
\begin{proof}[Proof of Inequality \eqref{eq:T_op}]
	For any $\vb \in \sphere^{p-1}$, we have
	\[
	\mathcal{T}(\ub)(\vb, \vb, \vb) = 3\sum_{i, j \in [p]} u_iv_iv_j^2 = 3\sum_{i \in [p]} u_iv_i\twonorm{\vb}^2 =  3\inprod{\ub}{\vb} \leq 3\twonorm{\ub}.
	\]
\end{proof}

\subsubsection{Proof of Theorem \ref{thm:tensor}} \label{sec:ensemble_proof_tensor}
With the previous analysis, we are ready to prove Theorem \ref{thm:tensor}. In the first place, we combine the ingredients in Section \ref{sec:concentration}. Recall that we split $n$ samples into two parts with size $n_1$ and $n_2$ for computing $m_0, \M_2$ and $\m_1, \M_3$ respectively. Putting \eqref{eq:weight_approx}, \eqref{eq:epsilon_0_bound}, \eqref{eq:epsilon_2_bound} together and using union bound, we have
\begin{equation} \label{eq:p_epsilon_2}
\Prob\left(\epsilon_2 \lesssim \log(\minprop n_1) \sqrt{\frac{p\log\left(12k/\delta\right)}{\minprop n_1}}\right) \geq \delta/2
\end{equation}
under condition $n_1 \gtrsim \frac{1}{\minprop}(\frac{k}{\delta}\vee p)$. Putting \eqref{eq:weight_approx}, \eqref{eq:epsilon_1_bound} and \eqref{eq:epsilon_3_bound} together leads to
\begin{equation} \label{eq:p_epsilon_3}
\Prob\left(\epsilon_3 \lesssim \frac{(k\vee\sqrt{p})\log^{3/2}(\minprop n_2)}{\sqrt{\sigma_k}^3} \sqrt{\frac{\log(12k/\delta)}{\minprop n_2}} \right) \geq \delta/2
\end{equation}
under conditions $n_2 \gtrsim k/(\minprop \delta)$ and $\epsilon_2 < \sigma_k/6$. In order to guarantee $\twonorm{\bbeta_j^{(0)} - \bbeta_j} \lesssim \varepsilon$ for all $j \in [k]$, using the error transfer inequality \eqref{eq:error_transfer} and noting that $\sigma_1 \leq 1, \opnorm{\overline{\M}_3} \leq 1$ under assumption $\max_{j \in [k]}\twonorm{\bbeta_j} = 1$, it is sufficient to require
\begin{equation} \label{eq:epsilon_23_condition}
\epsilon_2 \lesssim \sqrt{\sigma_k}^5\varepsilon,~~
\epsilon_3 \lesssim \varepsilon.
\end{equation}
The above condition on $\epsilon_2$ leads to $\epsilon_2 < \sigma_k/6$ for $\varepsilon \leq 1$. In addition, in order to let \eqref{eq:error_transfer} hold, $\epsilon_2, \epsilon_3$ have to satisfy condition \eqref{eq:require_ep_23}. This is implied by \eqref{eq:epsilon_23_condition} when $\varepsilon \lesssim 1/k$. Using the relationship between $\epsilon_2, \epsilon_3$ and $n_1, n_2$ in \eqref{eq:p_epsilon_2} and \eqref{eq:p_epsilon_3}, it is sufficient to require
\[
n_1 \gtrsim \frac{p\log(12k/\delta)\log^2(n_1)}{\minprop\sigma_k^5\varepsilon^2} \vee \frac{k}{\minprop \delta} ~~\text{and}~~ n_2 \gtrsim \frac{(k^2 \vee p)\log(12k/\delta)\log^3(n_2)}{\minprop\sigma_k^3\varepsilon^2} \vee \frac{k}{\minprop \delta},
\]
which concludes our proof.
\subsection{Proof of Alternating Minimization (Theorem \ref{thm:altm})} \label{sec:proof:thm:altm}
It is sufficient to show the linear error decay in one step. Then the error bound for each step $t$ can be obtained by induction. Without loss of generality, we focus on the first step $t = 0$. Also we assume $\pi(j) = j$ for simplicity. Let $B = n/T$ be the sample size in the first step. Let $\cA_j$ denote the index set of samples that are clustered to model $j$ in the label assignment step, namely
\[
\cA_j := \left\{i \in [B] ~\big|~ \abs{y_i - \inprod{\x_i}{\bbeta_j^{(0)}}} < \abs{y_i - \inprod{\x_i}{\bbeta_t^{(0)}}} \;\; \text{for all}\;\;t \ne j\right\}.
\]
We use $\cA_j^*$ to denote the set of samples that are truly generated from model $\bbeta_j$, namely
\[
\cA_j^* : =  \left\{i \in [B] ~\big|~ y_i = \inprod{\x_i}{\bbeta_j}\right\}.
\]
Introduce $\varepsilon_0$ as a shorthand for $\mathcal{E}(\{\bbeta_j^{(0)}\})$. According to our assumption, $\varepsilon_0 \lesssim \Delta/k^2$.

Let $\bSigma_j := \sum_{i \in \cA_j}\x_i\x_i^{\top}$ be the empirical covariance of samples in $\cA_j$. The updated estimate $\bbeta_j^{(1)}$ has the form
\[
\bbeta_j^{(1)} = \bSigma_j^{-1} \left(\sum_{i\in \cA_j}y_i\x_i\right).
\]
We thus obtain
\begin{align*}
\bbeta_j^{(1)} - \bbeta_j & = \bSigma_j^{-1} \left(\sum_{i\in \cA_j}y_i\x_i\right) - \bbeta_j = \bSigma_j^{-1} \left(\sum_{i\in \cA_j}y_i\x_i - \x_i\x_i^{\top}\bbeta_j\right) \\
& = \bSigma_j^{-1} \left(\sum_{t \in [k]} \sum_{i \in \cA_t^*\bigcap \cA_j}\x_i\x_i^{\top} (\bbeta_t - \bbeta_j\big) \right).
\end{align*}
By the Cauchy-Schwartz inequality, we obtain
\begin{align*}
\twonorm{\bbeta_j^{(1)} - \bbeta_j} & \leq \opnorm{\bSigma_j^{-1}} \cdot \twonorm{\sum_{t \in [k]} \sum_{i \in \cA_t^*\bigcap \cA_j}\x_i\x_i^{\top} (\bbeta_t - \bbeta_j\big)} \\
& \leq \underbrace{\opnorm{\bSigma_j^{-1}}}_{U_1} \cdot \underbrace{\left(\sum_{t \in [k]}\twonorm{ \sum_{i \in \cA_t^*\bigcap \cA_j}\x_i\x_i^{\top} (\bbeta_t - \bbeta_j\big) } \right)}_{U_2}
\end{align*}
Next we bound the two terms $U_1$ and $U_2$ respectively.
\vskip .1in
\paragraph{Bound of $U_1$.} First note that $\opnorm{\bSigma_j^{-1}} = 1/\sigma_{min}(\bSigma_j)$. We find that
\[
\sigma_{min}(\bSigma_j) = \sigma_{min}\left(\sum_{i \in \cA_j} \x_i\x_i^{\top}\right) \geq \sigma_{min}\left(\sum_{i \in \cA_j \cap \cA_j^*} \x_i\x_i ^{\top}\right).
\]
For $X \sim \mathcal{N}(\bm{0}, \Identity_p)$, we define event $\mathcal{E}_j$ as
\[
\mathcal{E}_j := \left\{\abs{\inprod{X}{\bbeta_j^{(0)} - \bbeta_j}} \leq \abs{\inprod{X}{ \bbeta_t^{(0)} - \bbeta_j}},~~ \text{for all} ~~ t \ne j \right\}.
\]
Accordingly, we have
\begin{equation}
\label{expectation_covariance}
\Exs[\x_i\x_i^{\top}\abs{i \in \cA_j \cap \cA_j^*] = \Exs[XX^{\top}}\mathcal{E}_j].
\end{equation}
To provide a lower bound of $\Prob(\mathcal{E}_j)$, we have 
\begin{align}
\Prob(\mathcal{E}_j)  & = 1 - \Prob(\mathcal{E}_j^c)  \geq 1 - \sum_{t \ne j} \Prob \{\langle X, \bbeta_j^{(0)} - \bbeta_j\rangle^2 \geq \langle X, \bbeta_t^{(0)} - \bbeta_j\rangle^2\} \notag\\
& \overset{(a)}{\geq} 1 - \sum_{t \ne j} \frac{\twonorm{\bbeta_j^{(0)} - \bbeta_j}}{\twonorm{\bbeta_t^{(0)} - \bbeta_j}} \overset{(b)}{\geq} 1 - (k-1)\frac{4}{7k^2} \geq 1 - \frac{4}{7k}. \label{eq:event_prob_lower}
\end{align}
Step $(b)$ holds because  since for all $t \ne j$,
\[
\frac{\twonorm{\bbeta_j^{(0)} - \bbeta_j}}{\twonorm{\bbeta_t^{(0)} - \bbeta_j}} \leq \frac{\twonorm{\bbeta_j^{(0)} - \bbeta_j}}{\twonorm{\bbeta_t - \bbeta_j} - \twonorm{\bbeta_t^{(0)} - \bbeta_t}} \leq \frac{\varepsilon_0}{\Delta - \varepsilon_0} \leq \frac{4}{7k^2},
\]
where the last step follows from condition $\varepsilon_0 \leq \Delta/(7k^2)$. Step $(a)$ in \eqref{eq:event_prob_lower} is from the next result, which is proved in Section \ref{sec:proof:lem:two_vector}.
\begin{lemma} \label{lem:two_vector}
	Let $X \sim \Gaussian(\bm{0}, \Identity_p)$. For any two fixed vectors $\ub, \vb \in \mathbb{R}^{p}$, we define 
	\[
	\mathcal{E} := \left\{\abs{\langle X, \ub\rangle} \leq \abs{\langle X, \vb\rangle}\right\}.
	\]
	We have that when $\twonorm{\ub} > \twonorm{\vb}$,
	\[
	\Prob(\mathcal{E}) 
	\leq \frac{\twonorm{\vb}}{\twonorm{\ub}}.
	\]
\end{lemma}
The next result, proved in Section \ref{sec:proof:lem:trace_unchange}, establishes the spectral structure of the covariance matrix of $X~\big|~\mathcal{E}_j$.
\begin{lemma} [Conditional Spectral Structure] \label{lem:trace_unchange}
	Let $X \sim \Gaussian(\bm{0}, \Identity_p)$. For any $k$ fixed vectors $\ub_1,...,\ub_k \in \mathbb{R}^{p}$, we define event
	\[
	\mathcal{E} := \left\{\abs{\langle X, \ub_1\rangle} \leq \abs{\langle X, \ub_j\rangle}, \;\;\text{for all} \;\;j \in [k]\right\}.
	\]
	When $\Prob(\mathcal{E}) > 0$, we have 
	\[
	\sigma_{max}\left(\Exs[XX^{\top}~\big|~ \mathcal{E}]\right) \leq k.
	\]
	and
	\begin{equation} \label{eq:tmp8}
	\sigma_{min}\left(\Exs[XX^{\top}|\mathcal{E}]\right) \geq \frac{1 - k (1 - \Prob(\mathcal{E}))}{\Prob(\mathcal{E})}.
	\end{equation}
\end{lemma}
The above result suggests that
\begin{equation} \label{eq:min_eigen}
\sigma_{min}(\Exs[XX^{\top}|\mathcal{E}_j]) \geq \frac{1 - k\cdot [1 - \Prob(\mathcal{E}_j)]}{\Prob(\mathcal{E}_j)} \geq \frac{3}{7\Prob(\mathcal{E}_j)} \geq \frac{3}{7}.
\end{equation}

Next we will show $\sigma_{min}(\sum_{i \in \cA_j \cap \cA_j^*}\x_i\x_i^{\top})$ is close to its expected value $|\cA_j \cap \cA_j^*|\sigma_{min}(\Exs[XX^{\top}|\mathcal{E}_j])$. First, we prove $|\cA_j \cap \cA_j^*|$ is large enough. As $\Prob(\mathcal{E}_j) \geq 1 - 4/(7k^2) \geq 1/2$, we have $\Exs[|\cA_j\cap \cA_j^*|] \geq \Exs[\frac{1}{2}|\cA_j^*|] = \frac{1}{2}w_jB$. Therefore, $|\cA_j\cap \cA_j^*|$ is summation of $B$ independent Bernoulli random variable with success probability at least $\omega_j/2$. Then we have
\begin{equation} \label{size_bound}
\Prob\left(|\cA_j \cap \cA_j^*| \leq \frac{1}{4}\omega_jB\right) \leq \Prob\left(\bigg| |\cA_j \cap \cA_j^*| - \Exs[|\cA_j\cap \cA_j^*|]\bigg| \geq \frac{1}{4}\omega_jB\right) \leq 2e^{-C\omega_jB} \leq 2e^{-C\minprop B},
\end{equation}
where the second step follows from Lemma \ref{lem:sum_bern} and $C$ is some constant. Conditioning on the event $|\cA_j \cap \cA_j^*| \geq \omega_jB/4$, we obtain $|\cA_j \cap \cA_j^*| \gtrsim p$ when $B \gtrsim p/\minprop$. 

Note that $X$ is sub-Gaussian random vector. Part (a) of Lemma \ref{lem:sub-gaussian} shows that $X$ is still sub-Gaussian vector conditioning on $\mathcal{E}_j$. Using the conclusion $|\cA_j \cap \cA_j^*| \gtrsim p$, concentration result of sub-Gaussian in \eqref{eq:concentration_scalar_2} (setting $t = 1/7$ and $K$ to be a constant) yields that, for some constant $C$,
\begin{equation} \label{eig_bound}
\Prob \left( \opnorm{\frac{1}{|\cA_j\cap \cA_j^*|} \sum_{i \in \cA_j\cap \cA_j^*}\x_i\x_i^{\top} - \Exs[XX^{\top}|\mathcal{E}_j]} \geq \frac{1}{7} \right) \leq 2e^{-C\omega_jB} \leq 2e^{-C\minprop B}.
\end{equation}
Putting \eqref{size_bound} and \eqref{eig_bound} together and using Weyl's theorem, we have that with probability at least $1 - 4e^{-C'\minprop B}$, 
\[
\sigma_{min}\left(\sum_{i \in \cA_j\cap \cA_j^*} \x_i\x_i^{\top}\right) \geq |\cA_j\cap \cA_j^*|\cdot \left(\sigma_{min}(\Exs[XX^{\top}~\big|~\mathcal{E}_j]) - \frac{1}{7}\right) \geq \frac{1}{4}\omega_jB \cdot \frac{2}{7} = \frac{1}{14}\omega_jB,
\]
We thus obtain 
\begin{equation} \label{eq:U1}
\Prob\left(U_1 \geq 14/(w_jB)\right) \leq 4e^{-C'\minprop B}.
\end{equation}
\vskip .1in
\paragraph{Bound of $U_2$.} Recall that 
\[
U_2 = \sum_{t \ne j} \twonorm{\sum_{\cA_t^*\cap \cA_j} \x_i\x_i^{\top}(\bbeta_t - \bbeta_j)}. 
\]
We will bound every term with different $t$ separately. Note that for any vector $\x \in \real^{p}$ and positive semidefinite matrix $\A \in \real^{p\times p}$, we have
\[
\twonorm{\A\x}^2 \leq \sigma_{max}(\A)\x^{\top}\A\x.
\]
Introduce $\Q_{t} = \sum_{\cA_t^*\cap \cA_j} \x_i\x_i^{\top}$. We find
\begin{align*}
\twonorm{\sum_{\cA_t^*\cap \cA_j} \x_i\x_i^{\top}(\bbeta_t - \bbeta_j)}^2 & \leq \sigma_{max}(\Q_{t}) \sum_{i \in \cA_t^*\cap \cA_j} (\bbeta_t - \bbeta_j)^{\top}\x_i\x_i^{\top}(\bbeta_t - \bbeta_j) \\
& = \sigma_{max}(\Q_t)\sum_{i \in \cA_t^*\cap \cA_j} \langle\x_i, \bbeta_t - \bbeta_j \rangle^2 \\
& \leq 2\sigma_{max}(\Q_t)\sum_{i \in \cA_t^*\cap \cA_j} \big(\langle\x_i, \bbeta_t - \bbeta_j^{(0)} \rangle^2 + \langle\x_i, \bbeta_j^{(0)} - \bbeta_j \rangle^2\big) \\
& \overset{(a)}{\leq} 2\sigma_{max}(\Q_t)\sum_{i \in \cA_t^*\cap \cA_j}\big(\langle\x_i, \bbeta_t - \bbeta_t^{(0)} \rangle^2 + \langle\x_i, \bbeta_j^{(0)} - \bbeta_j \rangle^2\big) \\
& \leq 4\sigma^2_{max}(\Q_t) \cdot \varepsilon_0^2,
\end{align*}
where step $(a)$ follows from the fact that for each $i \in \cA_t^*\cap \cA_j$, $\inprod{\x_i}{\bbeta_t - \bbeta_j^{(0)}}^2 \leq \inprod{\x_i}{\bbeta_t - \bbeta_t^{(0)}}^2$ due to the label assignment rule. Accordingly,
\begin{equation} \label{eq:qt}
U_2 \leq 2\sum_{t \ne j}\sigma_{max}(\Q_t)\varepsilon_0.
\end{equation}

It remains to bound $\sigma_{max}(\Q_t)$.
For each $t$, define 
\[
\cA_{j}^{t} : = \big\{i \in \cA_t^* ~\big|~ \abs{\inprod{\x_i}{\bbeta_t - \bbeta_j^{(0)}}} \leq \abs{\inprod{\x_i}{\bbeta_t - \bbeta_t^{(0)}}} \big\}
\]
as the set of samples that are generated from model $t$, but have smaller reconstruction error in $\bbeta_j^{(0)}$ compared to $\bbeta_t^{(0)}$.
We have $\cA_j \cap \cA_t^* \subseteq \cA_{j}^{t}$, which leads to
\begin{equation}
\label{max<max}
\sigma_{max}(\Q_t) \leq \sigma_{max}(\sum_{i \in \cA_{j}^{t}} \x_i\x_i^{\top}).
\end{equation}
In parallel, for $X \sim \Gaussian(\bm{0},\Identity_p)$, define 
\[
\mathcal{E}_{j}^{t} = \{\abs{\inprod{X}{\bbeta_t - \bbeta_j^{(0)}}} \leq \abs{\inprod{X}{\bbeta_t - \bbeta_t^{(0)}}}\}.
\]
Let $\overline{\varepsilon}_0$ be an upper bound of $\varepsilon_0$.
\[
\Exs[|\cA_{j}^{t}|] = \Exs[|\cA_t^*|]\cdot\Prob(\mathcal{E}_{j}^{t}) = \omega_tB\cdot \Prob(\mathcal{E}_{j}^{t}) \leq \omega_tB\frac{\twonorm{\bbeta_t - \bbeta_t^{(0)}}}{\twonorm{\bbeta_t - \bbeta_j^{(0)}}} \leq \omega_tB\frac{\overline{\varepsilon}_0}{\Delta - \overline{\varepsilon}_0} \leq \frac{2\omega_tB\overline{\varepsilon}_0}{\Delta},
\]
where the first inequality follows from Lemma \ref{lem:two_vector}, and the last step holds when $\overline{\varepsilon}_0 \leq \Delta/2$. Note that $|\cA_{j}^{t}|$ is a summation of independent Bernoulli random variables with success probability at most $2\omega_t\overline{\varepsilon}_0/\Delta$. Then by Lemma \ref{lem:sum_bern}, we have
\begin{equation} \label{eq:size_bound}
\Prob\left(|\cA_{j}^{t}|  - \Exs[|\cA_{j}^{t}|] \geq 2\omega_tB\overline{\varepsilon}_0/\Delta\right) \leq 2e^{-3\omega_tB\overline{\varepsilon}_0/(4\Delta)} \leq 2e^{-3\minprop B\overline{\varepsilon}_0/(4\Delta)}.
\end{equation}

Following Lemma \ref{lem:trace_unchange} (by setting $k = 2$), we have 
\[
\sigma_{max}\big(\Exs[XX^{\top}~\big|~\mathcal{E}_{j}^{t}]\big) \leq 2.
\]
Part (b) in Lemma \ref{lem:sub-gaussian} suggests that $X~\big|~\mathcal{E}_{j}^{t}$ is still sub-Gaussian random vector with constant Orlicz norm. According to the concentration result in Remark 5.40 of \cite{vershynin2010nonasym}, we have that with probability at least $1 - 2e^{-p}$,
\[
\sigma_{max}(\sum_{i \in \cA_{j}^{t}} \x_i\x_i^{\top})  \leq \abs{\cA_{j}^{t}}(2+ (\eta \vee \eta^2)),
\]
where $\eta \asymp \sqrt{p/\abs{\cA_{j}^{t}}}$. We thus have
\[
\sigma_{max}(\sum_{i \in \cA_{j}^{t}} \x_i\x_i^{\top}) \lesssim p \vee \abs{\cA_{j}^{t}} \lesssim p + \abs{\cA_{j}^{t}}. 
\]
Putting the above result, \eqref{eq:size_bound} and \eqref{max<max} together, and taking the union bound over all $t \ne j$ , we have that with probability at least $1 - ke^{-p} - 2ke^{-3\minprop B\overline{\varepsilon}_0/(4\Delta)}$,
\[
\sum_{t \ne j} \sigma_{max}(\Q_t) \lesssim \sum_{t \ne j} p + \abs{\cA_{j}^{t}} \lesssim 2kp + \sum_{t \ne j} \omega_tB\overline{\varepsilon}_0/\Delta \lesssim kp + B\overline{\varepsilon}_0/\Delta.
\]
Plugging the above result into \eqref{eq:qt} yields that for some constant $C$
\begin{equation} \label{eq:U2}
\Prob\left(U_2 \geq C(kp + B\overline{\varepsilon}_0/\Delta)\varepsilon_0\right) \leq  ke^{-p} + 2ke^{-3\minprop B\overline{\varepsilon}_0/(4\Delta)}.
\end{equation}

\vskip .1in
\paragraph{Ensemble.} Combining the bounds of $U_1$ and $U_2$, there exists a constant $C$ such that when $B \gtrsim p/\minprop$,
\[
\Prob\left(\twonorm{\bbeta_j^{(1)} - \bbeta_j} \geq \underbrace{C\frac{(kp + B\overline{\varepsilon}_0/\Delta)\varepsilon_0}{\omega_jB}}_{U} \right) \leq 4e^{-C'\minprop B} + 2ke^{-p} + 2ke^{-3\minprop B\overline{\varepsilon}_0/(4\Delta)}.
\]
Now we set $\overline{\varepsilon}_0 = \minprop \Delta/(4C)$. Then the condition $B \geq 4Ckp/\minprop$ leads to $U \leq \frac{1}{2}\varepsilon_0$. Accordingly
\[
\Prob\left(\twonorm{\bbeta_j^{(1)} - \bbeta_j} \geq \frac{1}{2}\varepsilon_0 \right) \leq 4e^{-C'\minprop B} + ke^{-p} + 2ke^{-3\minprop^2B/(16C)} \leq 2ke^{-p} + 4ke^{-C_1\minprop^2B} \leq \frac{\delta}{kT},
\]
where the last step follows from conditions $B \gtrsim \minprop^{-2}\log(8k^2T/\delta)$ and $p \geq \log(2k^2T/\delta)$. Taking union bound over all $j \in [k]$, we finish proving the error decay in the first iteration. Using the same calculation for all $T$ iterations and taking union bound concludes the proof.

\subsection{Proof of Lemma \ref{lem:exact recovery}} \label{sec:proof:lem:exact recovery}
For $X \sim \Gaussian(\bm{0}, \Identity_p)$, define event $\mathcal{E}_j$, which indicates the case that sample from model $j$ is correctly assigned label $j$, as
\[
\mathcal{E}_j := \left\{\abs{\inprod{X}{\widehat{\bbeta}_j - \bbeta_j}} \leq \abs{\inprod{X}{ \widehat{\bbeta}_t - \bbeta_j}},~~ \text{for all} ~~ t \ne j \right\}.
\]
According to \eqref{eq:event_prob_lower} in the proof of Theorem \ref{thm:altm}, we have
\[
\Prob(\mathcal{E}_j)  \geq 1 - \sum_{t \ne j} \frac{\twonorm{\bbeta_j^{(0)} - \bbeta_j}}{\twonorm{\bbeta_t^{(0)} - \bbeta_j}} \geq 1 - (k-1)\frac{\widehat{\varepsilon}}{\Delta - \widehat{\varepsilon}},
\] 
where $\widehat{\varepsilon} := \mathcal{E}(\{\widehat{\bbeta}_j\})$. Taking union bound over all $n$ samples, we have that the probability of correct assignment of all labels is at least 
\[
1 - n(k-1)\frac{\widehat{\varepsilon}}{\Delta - \widehat{\varepsilon}} \geq 1 - \delta/2,
\]
where the last step holds when $\widehat{\varepsilon} \leq \frac{\delta}{4nk}\Delta$. When $n \gtrsim \frac{p}{\minprop} \vee \frac{1}{\minprop}\log(k/\delta)$, using Lemma \ref{lem:sum_bern} and union bound, it is guaranteed that, with probability at least $1 - \delta/2$, each cluster has at least $p$ samples. Therefore, correct label assignment will lead to exact recovery.
\section{Proofs of Technical Lemmas}

\subsection{Proof of Lemma \ref{lem:concentration}} \label{sec:proof:lem:concentration}
Suppose $\Sb$ has an SVD $\Sb = \Ub \bSigma\Vb^{\top}$, where $\Ub \in \real^{p \times s}$, $\Vb \in \real^{p \times s}$ have orthonormal columns $\Ub^{\top}\Ub = \Vb^{\top}\Vb = \Identity_s$. We can always find  $\balpha \in \real^s$, $\bgamma \in \real^{p}$ such that $\bbeta = \Ub \balpha + \bgamma$, where $\Ub^{\top}\bgamma = \bm{0}$ and $\twonorm{\balpha}^2 + \twonorm{\bgamma}^2 = 1$. We let $X \sim \Gaussian(\bm{0}, \Identity_p)$, $Y = \inprod{X}{\bbeta}$.
\vskip .1in
\paragraph{Proof of Inequality \eqref{eq:concentration_1}.} We find
\begin{align}
\twonorm{\Sb^{\top}\left(\m_1 - \overline{\m}_1\right)} & = \twonorm{\Vb \bSigma \Ub^{\top} \left(\frac{1}{n}\sum_{i=1}^n y_i^3\x_i - \Exs\left[Y^3X\right]\right)} \notag \\
& = \twonorm{\bSigma \left(\frac{1}{n}\sum_{i=1}^n y_i^3 \Ub^{\top}\x_i - \Exs\left[Y^3 \Ub^{\top}X\right]\right)} \notag \\
& \leq \opnorm{\Sb}\cdot\twonorm{ \frac{1}{n}\sum_{i=1}^n y_i^3 \Ub^{\top}\x_i - \Exs\left[Y^3 \Ub^{\top}X\right]} \notag\\
& = \opnorm{\Sb} \cdot\twonorm{\frac{1}{n}\sum_{i=1}^n (\inprod{\balpha}{\widetilde{\x}_i} + z_i)^3 \widetilde{\x}_i - \Exs\left[(\inprod{\balpha}{\widetilde{X}} + Z)^3 \widetilde{X}\right]}, \label{eq:X_tilde}
\end{align}
where we let $\widetilde{X} := \Ub^{\top}X$, $Z := \inprod{\bgamma}{X}$, and $\{(z_i, \widetilde{\x}_i)\}_{i=1}^{n}$ are $n$ independent samples of $(\widetilde{X},Z)$. Thanks to the rotation invariance of Gaussian, we have $\widetilde{X} \sim \Gaussian(\bm{0}, \Identity_s)$ and $Z \sim \Gaussian(0, \twonorm{\bgamma}^2)$. Moreover, $\widetilde{X}$ and $Z$ are independent since $ \Ub^{\top}\bgamma = \bm{0}$.

For any $\tau_1, \tau_2 > 1$, define events 
\begin{equation} \label{eq:def_E}
\mathcal{E} := \left\{|\langle \balpha, \widetilde{X}\rangle| \leq \tau_1, |Z| \leq \tau_2\right\},~~ \mathcal{E}_n := \left\{|\langle \balpha, \widetilde{\x}_i\rangle| \leq \tau_1, |z_i| \leq \tau_2, ~\text{for all}~ i \in [n]\right\}.
\end{equation}
We have
\begin{align*}
& \twonorm{ \frac{1}{n}\sum_{i=1}^n (\langle \balpha, \widetilde{\x}_i\rangle + z_i)^3 \widetilde{\x}_i - \Exs\left[(\langle \balpha, \widetilde{X}\rangle + Z)^3 \widetilde{X}\right] } \\
& \leq \underbrace{\twonorm{\frac{1}{n}\sum_{i=1}^n (\langle \balpha, \widetilde{\x}_i\rangle + z_i)^3 \widetilde{\x}_i - \Exs\left[(\langle \balpha, \widetilde{X}\rangle + Z)^3 \widetilde{X} ~\big|~ \mathcal{E}\right] }}_{d_1} \\
& ~~~~~ + \underbrace{\twonorm{\Exs\left[(\langle \balpha, \widetilde{X}\rangle + Z)^3 \widetilde{X} ~\big|~ \mathcal{E}\right]  - \Exs\left[(\langle \balpha, \widetilde{X}\rangle + Z)^3 \widetilde{X} \right]}}_{d_2}. \\
\end{align*}
For term $d_2$, using \eqref{eq:mean_dev_1} in Lemma~\ref{lem:cond_mean_dev} by replacing $(a, b, \tau_1,\tau_2)$ in the statement with 
\[
(\twonorm{\balpha}, \twonorm{\bgamma}, \tau_1/\twonorm{\balpha}, \tau_2/\twonorm{\bgamma}),
\]
we obtain
\begin{align*}
d_2 & \leq \tau_1\left(\frac{\tau_1}{\twonorm{\balpha}}e^{-\frac{\tau_1^2}{2\twonorm{\balpha}^2}} + \frac{\tau_2}{\twonorm{\bgamma}}e^{-\frac{\tau_2^2}{2\twonorm{\bgamma}^2}}\right) \leq \tau_1(\tau_1e^{-\tau_1^2/2} + \tau_2e^{-\tau_2^2/2}),
\end{align*}
where the last step follows from the fact that function $xe^{-x^2/2}$ is monotonically decreasing on $x \geq 1$. To ease notation, we let 
\begin{equation} \label{eq:bounded_RV}
\widetilde{X}' \sim X ~|~ \mathcal{E}, ~~ Z' \sim Z~|~\mathcal{E}.
\end{equation}
Suppose $\{(\widetilde{\x}_i', z_i')\}_{i=1}^n$ are independent samples of $(\widetilde{X}', Z')$. We observe that
\[
\Prob(d_1 \geq t) \leq \Prob\left(\twonorm{\frac{1}{n}\sum_{i=1}^n (\langle \balpha, \widetilde{\x}_i'\rangle + z_i')^3 \widetilde{\x}_i' - \Exs\left[(\langle \balpha, \widetilde{X}'\rangle + Z')^3 \widetilde{X}' \right]} \geq t\right) + \Prob(\mathcal{E}_n^c).
\]
Since $|\langle \balpha, X'\rangle + Z'| \leq \tau_1 + \tau_2$, $(\langle \balpha, \widetilde{X}'\rangle + Z')^3 \widetilde{X}'$ is sub-Gaussian random vector with Orlicz norm 
\[
\left\|(\langle \balpha, \widetilde{X}'\rangle + Z')^3 \widetilde{X}'\right\|_{\psi_2} \lesssim (\tau_1 + \tau_2)^{3}.
\]
By concentration result \eqref{eq:concentration_scalar_1} in Lemma \ref{lem:concentration_subgaussian_vector}, we have that for some constants $C_1, C_2$, condition $n \geq C_1s(\tau_1 + \tau_2)^6/t^2$ leads to
\[
\Prob\left(\twonorm{\frac{1}{n}\sum_{i=1}^n (\inprod{\balpha}{\widetilde{\x}_i'} + z_i')^3 \widetilde{\x}_i' - \Exs\left[(\inprod{\balpha}{\widetilde{X}'} + Z')^3 \widetilde{X}' \right] } \geq t\right) \leq e^{-C_2nt^2/(\tau_1 + \tau_2)^6}.
\]
Meanwhile, the variance of $\inprod{\balpha}{\widetilde{X}}$ and $Z$ are both at most $1$. We thus obtain 
\begin{equation} \label{eq:Ec_bound}
\Prob(\mathcal{E}_n^c) \leq ne^{1-\tau_1^2} + ne^{1-\tau_2^2}
\end{equation} by using Gaussian tail bound and union bound. Accordingly, 
\[
\Prob(d_1 \geq t) \leq e^{-c_2nt^2/(\tau_1 + \tau_2)^6} + ne^{1-\tau_1^2} + ne^{1-\tau_2^2}.
\]
Setting $\tau_1 = \tau_2 = C\sqrt{\log n}$ for sufficiently large constant $C$ and $t \asymp (\tau_1+ \tau_2)^3\sqrt{1/n}\left(\sqrt{\log(\frac{2}{\delta})}\vee\sqrt{s}\right)$, we have $d_2 \lesssim 1/n$ and $\Prob(d_1 \geq t) \leq \delta/2 + 1/n$. Requiring $n \gtrsim 2/\delta$ gives our result.

\vskip .1in
\paragraph{Proof of Inequality \eqref{eq:concentration_2}.} We find
\begin{align*}
\opnorm{\Sb^{\top} \left( \M_2 - \overline{\M}_2\right)\Sb} & = \opnorm{\Vb \bSigma \Ub^{\top} \left(\frac{1}{n}\sum_{i \in [n]}y_i^2\x_i\x_i^{\top} - \Exs\left[Y^2XX^{\top}\right] \right) \Ub \bSigma \Vb^{\top}} \\
& \leq \opnorm{\Sb}^2 \cdot \opnorm{\frac{1}{n}\sum_{i \in [n]}y_i^2\widetilde{\x}_i\widetilde{\x}_i^{\top} - \Exs\left[Y^2\widetilde{X}\widetilde{X}^{\top}\right]},
\end{align*} 
where $\widetilde{\x}_i$ and $\widetilde{X}$ are defined according to \eqref{eq:X_tilde}. Using the $\mathcal{E}, \mathcal{E}_n$ defined in \eqref{eq:def_E}, we have
\begin{align*}
& \opnorm{\frac{1}{n}\sum_{i \in [n]}y_i^2\widetilde{\x}_i\widetilde{\x}_i^{\top} - \Exs\left[Y^2\widetilde{X}\widetilde{X}^{\top}\right] } \\
\leq & \underbrace{\opnorm{ \frac{1}{n}\sum_{i=1}^n (\inprod{ \balpha}{ \widetilde{\x}_i} + z_i)^2 \widetilde{\x}_i\widetilde{\x}_i^{\top} - \Exs\left[(\inprod{\balpha}{ \widetilde{X}} + Z)^2 \widetilde{X}\widetilde{X}^{\top} ~\big|~ \mathcal{E}\right] }}_{d_1} \\
& + \underbrace{ \opnorm{\Exs\left[(\inprod{\balpha}{\widetilde{X}} + Z)^2 \widetilde{X}\widetilde{X}^{\top} ~\big|~ \mathcal{E}\right]  - \Exs\left[(\inprod{ \balpha}{\widetilde{X}} + Z)^2 \widetilde{X}\widetilde{X}^{\top} \right]}}_{d_2}.
\end{align*}
Applying \eqref{eq:mean_dev_2} in Lemma~\ref{lem:cond_mean_dev} via setting $(a, b, \tau_1,\tau_2)$ in the statement to be 
\[
(\twonorm{\balpha}, \twonorm{\bgamma}, \tau_1/\twonorm{\balpha}, \tau_2/\twonorm{\bgamma})
\]
provides that
\begin{align*}
d_2 & \leq \frac{\tau_1^3}{\twonorm{\balpha}^3}e^{-\frac{\tau_1^2}{2\twonorm{\balpha}^2}} + \frac{\tau_1}{\twonorm{\balpha}}\frac{\tau_2}{\twonorm{\bgamma}}e^{-\frac{\tau_1^2}{2\twonorm{\balpha}^2}-\frac{\tau_2^2}{2\twonorm{\bgamma}^2}} \leq \tau_1^3e^{-\tau_1^2/2} + \tau_1\tau_2e^{-\tau_1^2/2 - \tau_2^2/2},
\end{align*}
where the last inequality follows from the fact that functions $x^3e^{-x^2/2}$, $xe^{-x^2/2}$ are monotonically decreasing when $x$ is sufficiently large.

We follow the same idea used before to bound $d_1$. Introduce $\widetilde{X}', Z'$ according to \eqref{eq:bounded_RV}. Then we obtain
\begin{equation} \label{eq:d1}
\Prob(d_1 \geq t) \leq \Prob\left(\opnorm{\frac{1}{n}\sum_{i=1}^n (\inprod{\balpha}{ \widetilde{\x}_i'} + z_i')^2 \widetilde{\x}_i'\widetilde{\x}_i'^{\top} - \Exs\left[(\inprod{\balpha}{\widetilde{X}'} + Z')^2 \widetilde{X}'\widetilde{X}'^{\top} \right] } \geq t\right) + \Prob(\mathcal{E}_n^c).
\end{equation}
Since $|\inprod{\balpha}{\widetilde{X}'} + Z'| \leq \tau_1 + \tau_2$, $(\inprod{\balpha}{\widetilde{X}'} + Z')\widetilde{X}'$ is sub-Gaussian random vector with norm $\mathcal{O}(\tau_1 + \tau_2)$. Applying \eqref{eq:concentration_scalar_1} in Lemma \ref{lem:concentration_subgaussian_vector}, we have that for $t \in (0, (\tau_1 + \tau_2)^2)$ and some constants $C_1,C_2$, the condition $n \geq C_1k(\tau_1 + \tau_2)^4/t^2$ yields
\[
\Prob\left(\opnorm{\frac{1}{n}\sum_{i=1}^n (\inprod{\balpha}{\widetilde{\x}_i'} + z_i')^2 \widetilde{\x}_i'\widetilde{\x}_i'^{\top} - \Exs\left[(\inprod{\balpha}{\widetilde{X}'} + Z')^2 \widetilde{X}'\widetilde{X}'^{\top} \right] } \geq t\right) \leq e^{-C_2nt^2/(\tau_1 + \tau_2)^4}.
\]
Plugging it back into \eqref{eq:d1}  and using the bound \eqref{eq:Ec_bound} of $\Prob(\mathcal{E}^c)$, we obtain
\[
\Prob(d_1 \geq t) \leq e^{-c_2nt^2/(\tau_1 + \tau_2)^4} + ne^{1-\tau_1^2} + ne^{1-\tau_2^2}.
\]
Choosing $\tau_1 = \tau_2 = C\sqrt{\log n}$ for sufficiently large constant $C$ and letting $t \asymp \frac{\log n}{\sqrt{n}}\left(\sqrt{\log\left(\frac{2}{\delta}\right)} \vee \sqrt{s}\right)$, we have that when $n \geq C'(1/\delta \vee s)$ for sufficiently large $C'$, it is guaranteed that $\Prob(d_1 \geq t) \leq \delta$ and $d_2 \lesssim 1/n$, which concludes the proof. 

\vskip .1in
\paragraph{Proof of Inequality \eqref{eq:concentration_3}.} Using the Cauchy-Schwartz inequality and the definitions of $\widetilde{\x}_i$ and $\widetilde{X}$ in \eqref{eq:X_tilde}, we have that
\[
\opnorm{\left(\M_3 - \overline{\M}_3\right)(\Sb,\Sb,\Sb)} \leq \opnorm{\Sb}^3 \cdot \opnorm{\frac{1}{n}\sum_{i \in [n]} y_i^3\widetilde{\x_i}\otimes\widetilde{\x_i}\otimes\widetilde{\x_i} - \Exs\left[ Y^3\widetilde{X}\otimes \widetilde{X}\otimes\widetilde{X}\right]}.
\]
Again, we use the event $\mathcal{E}, \mathcal{E}_n$ in \eqref{eq:def_E} to bound the operator norm. In detail, we have
\begin{align*}
& \opnorm{\frac{1}{n}\sum_{i \in [n]} y_i^3\widetilde{\x_i}^{\otimes 3} - \Exs\left[ Y^3\widetilde{X}^{\otimes 3}\right]} \leq \underbrace{\opnorm{ \frac{1}{n}\sum_{i=1}^n (\inprod{\balpha}{\widetilde{\x}_i} + z_i)^3 \widetilde{\x}_i^{\otimes 3} - \Exs\left[(\inprod{\balpha}{\widetilde{X}} + Z)^3 \widetilde{X}^{\otimes 3} ~\big|~ \mathcal{E}\right] }}_{d_1} \\
& ~~~ + \underbrace{\opnorm{\Exs\left[(\inprod{\balpha}{\widetilde{X}} + Z)^3 \widetilde{X}^{\otimes 3}~\big|~ \mathcal{E}\right]  - \Exs\left[(\inprod{\balpha}{ \widetilde{X}} + Z)^3 \widetilde{X}^{\otimes 3} \right]}}_{d_2}.
\end{align*}
Applying \eqref{eq:mean_dev_3} in Lemma~\ref{lem:cond_mean_dev} by setting $(a, b, \tau_1,\tau_2)$ in the statement to be $(\twonorm{\balpha}, \twonorm{\bgamma}, \tau_1/\twonorm{\balpha}, \tau_2/\twonorm{\bgamma})$, we obtain
\begin{align*}
d_2 & \leq \frac{\tau_1^5}{\twonorm{\balpha}^5}e^{-\frac{\tau_1^2}{2\twonorm{\balpha}^2}} + \frac{\tau_1^3}{\twonorm{\balpha}^3}\frac{\tau_2}{\twonorm{\bgamma}}e^{-\frac{\tau_1^2}{2\twonorm{\bgamma}^2}-\frac{\tau_2^2}{2\twonorm{\bgamma}^2}} \leq \tau_1^5e^{-\tau_1^2/2} + \tau_1^3\tau_2e^{-\tau_1^2/2 - \tau_2^2/2},
\end{align*}
where the last inequality follows from the fact that functions $x^5e^{-x^2/2}$, $x^3e^{-x^2/2}$, $xe^{-x^2/2}$ are monotonically decreasing when $x$ is sufficiently large. For term $d_1$, introducing the $\widetilde{X}'$, $Z'$ according to \eqref{eq:bounded_RV}, we have
\[ 
\Prob(d_1 \geq t) \leq \Prob\left(\opnorm{ \frac{1}{n}\sum_{i=1}^n (\inprod{\balpha}{ \widetilde{\x}_i'} + z_i')^3 \widetilde{\x}_i'^{\otimes 3} - \Exs\left[(\inprod{ \balpha}{\widetilde{X}'} + Z')^3 \widetilde{X}'^{\otimes 3} \right] } \geq t\right) + \Prob(\mathcal{E}_n^c).
\]
Note that $(\inprod{\balpha}{\widetilde{X}'} + Z')\widetilde{X}'$ is sub-Gaussian random vector with norm $\mathcal{O}(\tau_1 + \tau_2)$. Applying \eqref{eq:concentration_scalar_3} in Lemma \ref{lem:concentration_subgaussian_vector}, we find that for any $t \in (0, (\tau_1 + \tau_2)^3\sqrt{s})$ and constants $C_1,C_2$, condition $n \geq C_1(\tau_1 + \tau_2)^6s^2/t^2$ yields
\[
\Prob(d_1 \geq t) \leq e^{-c_2\frac{nt^2} {k^2(\tau_1 + \tau_2)^6}} + \Prob(\mathcal{E}_n^c) \leq e^{-c_2\frac{nt^2} {k^2(\tau_1 + \tau_2)^6}} + ne^{1-\tau_1^2} + ne^{1-\tau_2^2}.
\]
Setting $\tau_1 = \tau_2 = C\sqrt{\log n}$ for sufficiently large constant $C$, $t \asymp \frac{s\sqrt{\log n}^3}{\sqrt{n}}\sqrt{\log\left(\frac{2}{\delta}\right)}$, and assuming $n \gtrsim \max\{ s\log\left(\frac{2}{\delta}\right), 1/\delta\}$, we obtain that $\Prob(d_1 \leq t) \leq \delta$ and $d_2 \leq 1/n$, which concludes the proof.

\subsection{Proof of Lemma \ref{lem:two_vector}} \label{sec:proof:lem:two_vector}
For two vector $\ub,\vb$, we define angle $\alpha(\ub,\vb) \in [0, \pi]$ as
\[
\alpha(\ub,\vb) := \cos^{-1}\frac{(\ub - \vb)^{\top}(\ub+\vb)}{\twonorm{\ub+\vb}\cdot \twonorm{\ub-\vb}}.
\]
Without loss of generality, we assume $\ub,\vb$  live in the subspace spanned by $\e_1,\e_2$. We use $x_1,x_2$ to denote the first two coordinates of $X$. We can let
\[
x_1 = A \cos\theta, \;\; x_2 = A \sin\theta,
\] 
where $A$ is Rayleigh random variable, and $\theta$ is uniformly distributed over $[0,2\pi)$. Conditioning on $\mathcal{E}$, the range of $\theta$ is truncated to be $[\theta_0,\theta_0 +\alpha(\ub,\vb)]\cup [\theta_0+\pi,\theta_0+\pi+\alpha(\ub,\vb)]$, where $\theta_0$ depends on $\ub,\vb$. Therefore, we have
\[
\Prob(\mathcal{E}) = \frac{\alpha(\ub,\vb)}{\pi}.
\] 
If $\twonorm{\ub} > \twonorm{\vb}$,
\[
\cos[\alpha(\ub,\vb)] \geq \frac{\twonorm{\ub}^2 - \twonorm{\vb}^2}{\twonorm{\ub}^2 + \twonorm{\vb}^2} > 0.
\]
So we have $\alpha(\ub,\vb) \in [0, \pi/2]$. Using the fact that $\alpha < \frac{\pi}{2}\sin \alpha$ for any $\alpha \in [0,\pi/2]$, we have
\[
\Prob(\mathcal{E}) \leq \frac{1}{2}\sin[\alpha(\ub,\vb)] \leq \frac{\twonorm{\ub}\twonorm{\vb}}{\twonorm{\ub}^2 + \twonorm{\vb}^2} \leq \frac{\twonorm{\vb}}{\twonorm{\ub}}.
\]

\subsection{Proof of Lemma \ref{lem:trace_unchange}} \label{sec:proof:lem:trace_unchange}
Note that conditioning on $\mathcal{E}$ or $\mathcal{E}^c$ will not change the distribution of $\twonorm{X}$. We thus have 
\[
\Exs\left[\twonorm{X}^2~\big|~ \mathcal{E}\right] = \Exs\left[\twonorm{X}^2~\big|~ \mathcal{E}^c\right] = \Exs\twonorm{X}^2 = p.
\]
Hence,
\begin{equation}
\label{trace}
\text{Trace}\left(\Exs\left[XX^{\top}~\big|~ \mathcal{E}\right]\right) = p.
\end{equation}
Also note that $\Exs\left[XX^{\top}~\big|~ \mathcal{E}\right]$ and $\Exs\left[XX^{\top}~\big|~ \mathcal{E}^c\right]$ have at least $p-k$ eigenvalues that are $1$ since $\{\ub_1,\ldots,\ub_k\}$ spans a subspace with dimension at most $k$. Therefore we have
\[
\sigma_{max}\left(\Exs\left[XX^{\top}~\big|~ \mathcal{E}\right]\right) \leq \text{Trace}\left(\Exs\left[XX^{\top}~\big|~ \mathcal{E}\right]\right) - (p-k) \leq k.
\]
The above inequality also holds for $\sigma_{max}\left(\Exs\left[XX^{\top}~\big|~ \mathcal{E}^c\right]\right)$.
Note that
\[
\Identity_p =  \Exs[XX^{\top}] = \Exs\left[XX^{\top}~\big|~\mathcal{E}\right]\Prob(\mathcal{E}) + \Exs\left[ XX^{\top}~\big|~\mathcal{E}^c\right](1 - \Prob(\mathcal{E})).
\]
Suppose $\vb$ is the eigenvector that corresponds to the minimum eigenvalue of $\Exs\left[XX^{\top}~\big|~ \mathcal{E}\right]$. Therefore, we have
\begin{align*}
1 & = \vb^{\top}\Exs\left[XX^{\top}~\big|~\mathcal{E}\right]\vb \Prob(\mathcal{E}) + \vb^{\top}\Exs\left[ XX^{\top}~\big|~\mathcal{E}^c\right]\vb(1 - \Prob(\mathcal{E}))\\
& \leq \sigma_{min}\left(\Exs\left[XX^{\top}~\big|~\mathcal{E}\right]\right)\Prob(\mathcal{E}) + \vb^{\top}\Exs\left[ XX^{\top}~\big|~\mathcal{E}^c\right] \vb(1 -  \Prob(\mathcal{E})) \\
& \leq \sigma_{min}\left(\Exs\left[XX^{\top}~\big|~\mathcal{E}\right]\right)\Prob(\mathcal{E}) + k(1 -  \Prob(\mathcal{E})).
\end{align*}

\section{Auxiliary Results}
\begin{lemma}[Sum of Bernoulli Random Variables] \label{lem:sum_bern}Suppose $X_1,\ldots,X_n$ are $n$ independent Bernoulli random variables with $\Prob[X_1 = 0] = 1 - p$ and $\Prob[X_1 = 1] = p$. Let 
\[
\overline{X} = \frac{1}{n}\sum_{i \in [n]}X_i.
\]
For every $t > 0$, we have
\[
\Prob\left(\abs{\overline{X} - \Exs[\overline{X}]} \geq tp\right) \leq 2e^{-\frac{3t^2}{2(t+3)}np}.
\]
\end{lemma}
\begin{proof}
	We find that $X_1 - \Exs[X_1]$ has variance $p(1-p)$ and $|X_1 - \Exs[X_1]| \leq 1$. Using Bernstein's inequality, we have
	\[
	\Prob\left(|\overline{X} - \Exs[\overline{X}]| \geq tp\right) \leq 2e^{-\frac{nt^2p^2/2}{p(1-p)+tp/3}} \leq 2e^{-\frac{3t^2}{2(t+3)}np}.
	\]
\end{proof}

\begin{lemma}[Properties of Whitening Matrices, Lemma 6 in \cite{chaganty13}] \label{lem:whitening}Suppose $\A$ and $\widehat{\A}$ are both positive semidefinite matrices in $\real^{p\times p}$ with rank $k$. Let $\W, \widehat{\W} \in \real^{p\times k}$ be whitening matrices such that $\W^{\top}\A\W = \Identity_k$, $\widehat{\W}^{\top}\widehat{\A}\widehat{\W} = \Identity_k$. When $\alpha := \opnorm{\A - \widehat{\A}}/\sigma_k(\A) < 1/3$, we have
\begin{align*}
& \opnorm{\widehat{\W}} \leq 2\opnorm{\W},~~ \opnorm{\widehat{\W}^{\dagger}} \leq 2\opnorm{\W^{\dagger}}, \\
& \opnorm{\W - \widehat{\W}}  \leq 2\alpha\cdot\opnorm{\W}, ~~ \opnorm{\W^{\dagger} - \widehat{\W}^{\dagger}}  \leq 2\alpha\cdot\opnorm{\W^{\dagger}}.
\end{align*}
\end{lemma}

\begin{lemma}[Concentration of Sub-Gaussian Vectors] \label{lem:concentration_subgaussian_vector} Suppose $\x_1,\x_2, \ldots, \x_n \in \real^{p}$ are $n$ i.i.d. sub-Gaussian vectors with Orlicz norm $\|\x_1\|_{\psi_2} \leq K$. 
\begin{enumerate}
	\item There exist constants $C_i$ such that for every $t > 0$, when $n \geq C_1(K/t)^2p$,
	\begin{equation} \label{eq:concentration_scalar_1}
	\Prob\left( \twonorm{\frac{1}{n}\sum_{i \in [n]}\x_i - \Exs\left[\x_1\right]} \geq t\right)	 \leq e^{- C_2nt^2/K^2}.
	\end{equation}
	
	\item There exist constants $C_i$ such that for every $t \in (0, K^2)$, when $n \geq C_1(K^2/t)^2p$,
	\begin{equation} \label{eq:concentration_scalar_2}
	\Prob\left( \opnorm{\frac{1}{n}\sum_{i \in [n]}\x_i\x_i^{\top} - \Exs\left[\x_1\x_1^{\top}\right]} \geq t\right)	 \leq e^{- C_2nt^2/K^4}.
	\end{equation}
	
	\item There exist constants $C_i$ such that for every $t \in (0, K^3\sqrt{p})$, when $n \geq C_1(K^3/t)^2p^2$,
	\begin{equation} \label{eq:concentration_scalar_3}
	\Prob\left( \opnorm{\frac{1}{n}\sum_{i \in [n]}\x_i^{\otimes3} - \Exs\left[\x_1 ^{\otimes 3}\right]} \geq t\right)	 \leq e^{- C_2nt^2/(p^2K^6)}.
	\end{equation}
\end{enumerate}
\end{lemma}
\begin{proof}
	\noindent 
	\vskip .1in
	\noindent1. Note that
	\[
	\twonorm{\frac{1}{n}\sum_{i \in [n]}\x_i - \Exs\left[\x_1\right]} = \sup_{\ub \in \mathbb{S}^{p-1}} \left|\frac{1}{n}\sum_{i=1}^n \inprod{ \x_i}{\ub} - \Exs\left[\inprod{\x_i}{\ub}\right] \right|.
	\]
	
	Since $\x_i$ is sub-Gaussian vector, then for any fixed $\ub \in \mathbb{S}^{p-1}$, $\inprod{ \x_i}{\ub}$ is sub-Gaussian random variable with norm $K$. Therefore, $\inprod{ \x_i}{\ub} - \Exs\left[\inprod{ \x_i}{\ub}\right]$ is also sub-Gaussian with norm at most $2K$. By standard concentration of sub-Gaussianity, for some constant $C$, we obtain 
	\[
	\Prob\left(\left|\frac{1}{n}\sum_{i=1}^n \inprod{\x_i}{\ub} - \Exs\left[\inprod{ \x_i}{\ub}\right] \right| \geq t\right) \leq e^{1 - Cnt^2/K^2}.
	\]
	It is possible to construct an $\epsilon$-net $\mathcal{S}_{\epsilon}$ of $\mathbb{S}^{p-1}$ with size $|\mathcal{S}_{\epsilon}| \leq (1 + 2/\epsilon)^p$ (Lemma 5.2 in \cite{vershynin2010nonasym}). Applying probabilistic union bound leads to
	\[
	\Prob\left(\sup_{\ub \in \mathcal{S}_{\epsilon}} \left|\frac{1}{n}\sum_{i=1}^n \inprod{\x_i}{\ub} - \Exs\left[\inprod{ \x_i}{\ub}\right] \right| \geq t\right) \leq (1 + 2/\epsilon)^pe^{1 - Cnt^2/K^2}.
	\]
	For any $\z \in \mathbb{S}^{p-1}$, we can always find $\ub \in \mathcal{S}_{\epsilon}$ such that $\twonorm{\z - \ub} \leq \epsilon$. Then 
	\[
	\left|\frac{1}{n}\sum_{i=1}^n \inprod{ \x_i}{ \z} - \Exs\left[\inprod{\x_i}{\z}\right] \right| \leq \left|\frac{1}{n}\sum_{i=1}^n \inprod{\x_i}{\ub} - \Exs\left[\inprod{\x_i}{\ub}\right] \right|  + \twonorm{\ub - \z} \cdot\twonorm{\frac{1}{n}\sum_{i \in [n]}\x_i - \Exs\left[\x_1\right]}.
	\]
	Therefore, we obtain
	\begin{equation} \label{eq:sup_z}
	\sup_{\z \in \mathbb{S}^{p-1}} \left|\frac{1}{n}\sum_{i=1}^n \inprod{\x_i}{\z} - \Exs\left[\inprod{\x_i}{\z}\right] \right| \leq \frac{1}{1-\epsilon} \cdot\sup_{\ub \in \mathcal{S}_{\epsilon}}	\left|\frac{1}{n}\sum_{i=1}^n \inprod{\x_i}{\ub} - \Exs\left[\inprod{\x_i}{\ub}\right] \right|.
	\end{equation}
	Setting $\epsilon = 1/4$ and assuming $n \geq C'(K/t)^2p$ for sufficiently large constant $C'$ completes the proof.
	\vskip .1in
	\noindent 2. Refer to Theorem 5.39 in \cite{vershynin2010nonasym} for the proof.
	\vskip .1in
	\noindent3. Note that for any 3-way tensor $\T \in \real^{p\times p \times p}$ and two vectors $ \ub, \vb \in \real^{p}$ that satisfy $\twonorm{\ub - \vb} \leq \epsilon$, we have
	\begin{align*}
	\T(\ub, \ub, \ub) - \T(\vb, \vb, \vb) & = \T(\ub - \vb, \ub, \ub) + \T(\vb, \ub - \vb, \ub) + \T(\vb, \vb, \ub - \vb) \\
	& \leq 3\epsilon \cdot \sup_{\ab, \bb, \cb \in \sphere^{p-1}} \abs{\T(\ab, \bb, \cb)} \leq 27\epsilon \opnorm{\T},
	\end{align*}
	where the last inequality follows from Lemma~\ref{lem:general_opt_norm}. Constructing an $\epsilon$-net $\mathcal{S}_\epsilon$ on $\sphere^{p-1}$ and following similar idea in showing \eqref{eq:sup_z}, we obtain
	\[
	\opnorm{\frac{1}{n}\sum_{i \in [n]}\x_i^{\otimes 3} - \Exs\left[\x_1^{ \otimes3}\right]} \leq \frac{1}{1-27\epsilon}\sup_{\ub \in \mathcal{S}_\epsilon} \left| \frac{1}{n}\sum_{i \in [n]}\inprod{\x_i}{\ub}^3 - \Exs\left[\inprod{\x_1}{\ub }^3\right] \right|.
	\]
	Now we set $\epsilon = 1/54$, which leads to $|\mathcal{S}_\epsilon| \leq 109^p$. For any fixed $\ub \in \sphere^{p-1}$, $\inprod{\x_i}{\ub}$ is sub-Gaussian random variable with norm $K$. Using the concentration of cubes of sub-Gaussians (Lemma~\ref{lem:sum_cubes_subGaussian}) and applying union bound, we obtain
	\[
	\Prob\left( \sup_{\ub \in \mathcal{S}_\epsilon} \left| \frac{1}{n}\sum_{i \in [n]}\inprod{\x_i}{\ub}^3 - \Exs\left[\inprod{\x_1}{\ub}^3\right] \right| > CK^3\frac{\sqrt{p^3\log^3(109/\delta) + 2 p^2 \log^2(109/\delta) n}}{n} \right) \leq \delta
	\]
	for any $\delta \in (0,1)$ and some constant $C > 0$. Finally, for any $t \in (0, K^3\sqrt{p})$, setting $\delta = e^{-C'\frac{nt^2}{p^2K^6}}$, $n \geq C''(p/t)^2K^6$ for some constants $C',C''$ completes the proof.
\end{proof}

The next result shows a tail bound of a finite sum of sub-Gaussian random variables. A similar result is proved in the case of Gaussian in \cite{Hsu2012Gussian}. Here, we present our proof that can cover general sub-Gaussian distribution.
\begin{lemma} [Sum of Cubes of Sub-Gaussians] \label{lem:sum_cubes_subGaussian} Suppose $X_1, X_2, \ldots, X_n$ are $n$ i.i.d. sub-Gaussian random variables with Orlicz norm $\|X_1\|_{\psi_2} \leq K$. There exists an absolute constant $C$ such that for any $\delta \in (0,1)$,
	\[
	\Prob\left( \left| \frac{1}{n}\sum_{i=1}^nX_i^3 - \Exs\left[X_1^3\right] \right| > CK^3\frac{\sqrt{\log^3(1/\delta) + 2 \log^2(1/\delta) n}}{n} \right) \leq \delta.
	\]
\end{lemma}
\begin{proof}
	For any positive even integer $q$ and $t \in \mathbb{R}^{+}$, by Markov's inequality, we have
	\begin{align*}
	\Prob\left( \left| \frac{1}{n}\sum_{i=1}^nX_i^3 - \Exs\left[X_1^3\right] \right| > t\right) & = \Prob\left( \left( \frac{1}{n}\sum_{i=1}^nX_i^3 - \Exs\left[X_1^3\right] \right)^q > t^q\right) \\
	& \leq \frac{1}{t^q} \Exs\left[ \left( \frac{1}{n}\sum_{i=1}^nX_i^3 - \Exs\left[X_1^3\right] \right)^q\right].
	\end{align*}
	Let $X_1', X_2',\ldots, X_n'$ be another set of $n$ i.i.d. samples. We find 
	\begin{align*}
	\Exs\left[ \left( \frac{1}{n}\sum_{i=1}^nX_i^3 - \Exs\left[X_1^3\right] \right)^q\right] & = \Exs\left[ \left( \frac{1}{n}\sum_{i=1}^nX_i^3 - \frac{1}{n}\sum_{i=1}^n\Exs\left[X_i'^3\right] \right)^q\right] \\
	& \overset{(a)}{\leq} \Exs_{X_i, X_i'} \left[ \left(\frac{1}{n}\sum_{i=1}^n\left(X_i^3 - X_i'^3\right) \right)^q\right] \\
	& \overset{(b)}{\leq} \Exs_{X_i, X_i', \sigma_i} \left[ \left( \frac{1}{n}\sum_{i=1}^n\sigma_i\left(X_i^3 - X_i'^3\right) \right)^q\right] \\
	& \overset{(c)}{\leq} \Exs_{X_i, X_i', \sigma_i} \left[ 2^{q-1}\left( \frac{1}{n}\sum_{i=1}^n\sigma_iX_i^3\right)^q + 2^{q-1}\left( \frac{1}{n}\sum_{i=1}^n\sigma_iX_i'^3\right)^q\right] \\
	& = \left(\frac{2}{n}\right)^q \Exs_{X_i, \sigma_i} \left[ \left( \sum_{i=1}^n\sigma_iX_i^3\right)^q\right],
	\end{align*}
	where $(a)$ and $(c)$ follow from Jensen's inequality. In step $(b)$, we introduce Rademacher sequence $\sigma_1, \sigma_2,\ldots, \sigma_n$, i.e., $\Prob(\sigma_i = 1) = \Prob(\sigma_i = -1) = 0.5$. To ease notation, we let $Z_i := \sigma_i X_i$. So $\sigma_iX_i^3 = Z_i^3$ and $Z_i$ is still sub-Gaussian with norm $K$. It thus remains to bound $\Exs\left[\left(\sum_{i = 1}^n Z_i^3\right)^q\right]$. Note that $Z_i$ has symmetric distribution around $0$, so $\Exs[Z_i^a] = 0$ for any odd integer $a$. Accordingly, we have
	\[
	\Exs\left[\left(\sum_{i = 1}^n Z_i^3\right)^q\right] = \sum_{q_1 + \ldots + q_n = q/2} \prod_{i=1}^n\Exs\left[Z_i^{6q_i}\right] \leq \sum_{q_1 + \ldots + q_n = q/2} \prod_{i=1}^n (K\sqrt{6q_i})^{6q_i},
	\]
	where the last inequality follows from the basic property that if $X$ is sub-Gaussian random variable with norm $K$, then $(\Exs\left[|X|^q\right])^{1/q} \leq K\sqrt{q}$ for all $q > 1$. Since all $q_i \leq q/2$, we have
	\[
	\Exs\left[\left(\sum_{i = 1}^n Z_i^3\right)^q\right]  \leq {q/2 + n - 1 \choose q/2} \left(K\sqrt{3q}\right)^{3q} \leq \left(\frac{(q/2 + n - 1)e}{q/2}\right)^{q/2} \left(K\sqrt{3q}\right)^{3q}.
	\]	
	Putting all pieces together, we have
	\[
	\Prob\left( \left| \frac{1}{n}\sum_{i=1}^nX_i^3 - \Exs\left[X_1^3\right] \right| > t\right) \leq \left(\frac{18K^3q\sqrt{q+2n}}{nt}\right)^q.
	\]
	Setting $q = \lceil \log(1/\delta)\rceil$, $t = 18eK^3\frac{\sqrt{\log^3(1/\delta) + 2 \log^2(1/\delta) n}}{n}$ completes the proof.
\end{proof}

\begin{lemma} \label{lem:general_opt_norm}
	For any symmetric 3-way tensor $\T \in \real^{p\times p \times p}$, 
	\[
	\sup_{\ub, \vb, \w \in \sphere^{p-1}} \abs{\T(\ub, \vb, \w)} \leq 9 \opnorm{\T}.
	\]
\end{lemma}
\begin{proof}
	For any $\ub, \vb, \w \in \sphere^{p-1}$, we have
	\begin{align*}
	2\abs{\T(\ub, \vb, \w)} & = \abs{ \T(\ub, \vb, \w) + \T(\vb, \ub, \w) } = \abs{ \T(\ub + \vb, \ub + \vb, \w) - \T(\ub, \ub, \w) - \T(\vb, \vb, \w) } \\
	& \leq \abs{ \T(\ub + \vb, \ub + \vb, \w)} + \abs{ \T(\ub, \ub, \w)} + \abs{ \T(\vb, \vb, \w) } \leq 6 \sup_{\ab, \bb \in \sphere^{p-1}} \abs{\T( \ab, \ab, \bb)},
	\end{align*}
	where the first step holds because $\T$ is symmetric. Moreover, for any $\ub, \vb \in \sphere^{p-1}$, we have
	\begin{align*}
	6{\T(\ub, \ub, \vb)} = \abs{\T(\ub + \vb, \ub + \vb,\ub + \vb)  + \T(\vb - \ub, \vb - \ub,\vb - \ub) - 2 \T(\vb, \vb, \vb)} \leq 18\opnorm{\T}.
	\end{align*}
	Combining the above two inequalities leads to
	\[
	\sup_{\ub, \vb, \w \in \sphere^{p-1}} \abs{\T(\ub, \vb, \w)} \leq 3\sup_{\ub, \vb \in \sphere^{p-1}} \abs{\T(\ub, \ub, \vb)} \leq 9\opnorm{\T}.
	\]
\end{proof}

\begin{lemma}[Conditional Mean Deviation] \label{lem:cond_mean_dev} Let $X \sim \Gaussian(\bm{0}, \Identity_p)$, $Z \sim \Gaussian(0,1)$, and assume $X$ and $Z$ are independent. For any $\tau_1, \tau_2 \geq 1, \vb \in \sphere^{p-1}$, we define event $\mathcal{E} :=  \left\{\abs{\inprod{X}{\vb}} \leq \tau_1, |Z| \leq \tau_2\right\}$. For any $a, b > 0$, let $Y := a\cdot \inprod{X}{\vb} + b\cdot Z$. There exists constant $C$ such that the following inequalities hold.
\begin{enumerate}
	\item
	\begin{equation} \label{eq:mean_dev_1}
	\twonorm{\Exs\left[Y^3X ~\big|~\mathcal{E}\right] - \Exs\left[Y^3X\right]} \leq C (a^3 + ab^2) \left(\tau_1^3e^{-\tau_1^2/2} + \tau_1\tau_2e^{-\tau_1^2/2-\tau_2^2/2}\right).
	\end{equation}
	
	\item
	\begin{equation} \label{eq:mean_dev_2}
	\opnorm{\Exs\left[Y^2X\otimes X ~\big|~\mathcal{E}\right] - \Exs\left[Y^2X\otimes X\right]} \leq C(a^2 + b^2)\left(\tau_1^3e^{-\tau_1^2/2} + \tau_1\tau_2e^{-\tau_1^2/2-\tau_2^2/2}\right).
	\end{equation}
	
	\item
	\begin{equation} \label{eq:mean_dev_3}
	\opnorm{\Exs\left[Y^3X\otimes X \otimes X ~\big|~\mathcal{E}\right] - \Exs\left[Y^2X\otimes X \otimes X\right]} \leq C(a^3 + ab^2)\left(\tau_1^5e^{-\tau_1^2/2} + \tau_1^3\tau_2e^{-\tau_1^2/2-\tau_2^2/2}\right).
	\end{equation}
\end{enumerate}	
\end{lemma}
\begin{proof}
	
	\vskip .1in
	\noindent1. There exists $\ub \in \sphere^{p-1}$ such that
	\[
	\delta_1 := \twonorm{\Exs\left[Y^3X ~\big|~\mathcal{E}\right] - \Exs\left[Y^3X\right]} = \Exs\left[Y^3\inprod{X}{\ub} ~\big|~\mathcal{E}\right] - \Exs\left[Y^3\inprod{X}{\ub}\right]
	\]
	Due to the rotation invariance of spherical Gaussian vector, without loss of generality, we can simply assume $\vb = \e_1$ and $\ub = c\e_1 + d\e_2$, where $c^2 + d^2 = 1$. Let $X = \left(X_1, X_2,\ldots, X_p\right)^{\top}$. Using the symmetricity of $X_1, Z, X_2$ when conditioning on $\mathcal{E}^c$, we have
	\[
	\Exs\left[Y^3\inprod{X}{\ub} ~\big|~\mathcal{E}\right] = \Exs\left[(aX_1 + bZ)^3(cX_1 + dX_2) ~\big|~\mathcal{E}\right] = \Exs\left[a^3cX_1^4 + 3ab^2cX_1^2Z^2 ~\big|~ \mathcal{E}\right] \lesssim a^3|c| + ab^2|c|.
	\]
	Note that $X_1, Z, X_2$ are also symmetric when conditioning on $\mathcal{E}^c$, we thus obtain
	\begin{align*}
	\Exs\left[Y^3\inprod{X}{\ub} ~\big|~\mathcal{E}^c\right] \Prob(\mathcal{E}^c) & = \Exs\left[a^3cX_1^4 + 3ab^2cX_1^2Z^2 ~\big|~ \mathcal{E}^c\right]\Prob(\mathcal{E}^c) \\
	& \lesssim a^3|c|\tau_1^3e^{-\tau_1^2/2} + ab^2|c|\tau_1\tau_2e^{-\tau_1^2/2 - \tau_2^2/2},
	\end{align*}
	where the last inequality follows from Lemma~\ref{lem:cond_moment}. Now we turn to $\delta_1$. We find
	\begin{align*}
	\delta_1 & = \Exs\left[Y^3\inprod{X}{\ub} ~\big|~\mathcal{E}\right] - \Exs\left[Y^3\inprod{X}{\ub} ~\big|~ \mathcal{E}\right]\Prob(\mathcal{E}) -  \Exs\left[Y^3\inprod{X}{\ub} ~\big|~ \mathcal{E}^c\right]\Prob(\mathcal{E}^c) \\
	& \leq \left|\Exs\left[Y^3\inprod{X}{\ub} ~\big|~\mathcal{E}\right]\right|\Prob(\mathcal{E}^c) + \left|\Exs\left[Y^3\inprod{ X}{\ub} ~\big|~ \mathcal{E}^c\right]\Prob(\mathcal{E}^c)\right|. \\
	& \lesssim (a^3|c| + 3ab^2|c|)e^{-\tau_1^2/2 - \tau_2^2/2} + a^3|c|\tau_1^3e^{-\tau_1^2/2} + ab^2|c|\tau_1\tau_2e^{-\tau_1^2/2 - \tau_2^2/2} \\
	& \lesssim (a^3 + ab^2) \left(\tau_1^3e^{-\tau_1^2/2} + \tau_1\tau_2e^{-\tau_1^2/2-\tau_2^2/2}\right).
	\end{align*}
	\vskip .1in
	\noindent2. There exists $\ub \in \sphere^{p-1}$ such that
	\[
	\delta_2 := \opnorm{\Exs\left[Y^2X\otimes X ~\big|~\mathcal{E}\right] - \Exs\left[Y^2X\otimes X\right]} = \Exs\left[Y^2\inprod{ X}{\ub}^2 ~\big|~\mathcal{E}\right] - \Exs\left[Y^2\inprod{X}{\ub}^2\right].
	\]
	Using the same simplification argument in (a), we have
	\begin{align*}
	\Exs\left[Y^2\inprod{X}{\ub}^2 ~\big|~\mathcal{E}\right] & = \Exs\left[(aX_1 + bZ)^2(cX_1 + dX_2)^2 ~\big|~\mathcal{E}\right] \\
	& = \Exs\left[a^2c^2X_1^4 + b^2c^2X_1^2Z^2 + a^2d^2X_1^2X_2^2 + b^2d^2X_2^2Z^2 ~\big|~ \mathcal{E}\right] \lesssim a^2 + b^2.
	\end{align*}
	Applying Lemma~\ref{lem:cond_moment} again leads to
	\begin{align*}
	\Exs\left[Y^2\inprod{X}{\ub}^2 ~\big|~\mathcal{E}^c\right] \Prob(\mathcal{E}^c) & \lesssim a^2c^2\tau_1^3e^{-\tau_1^2/2} + b^2c^2\tau_1\tau_2e^{-\tau_1^2/2-\tau_2^2/2} + a^2d^2\tau_1e^{-\tau_1^2/2} + b^2d^2\tau_2e^{-\tau_2^2/2} \\
	& \lesssim (a^2 + b^2)\left(\tau_1^3e^{-\tau_1^2/2} + \tau_1\tau_2 e^{-\tau_1^2/2-\tau_2^2/2}\right).
	\end{align*}
	
	Overall, we have
	\begin{align*}
	\delta_2 & = \Exs\left[Y^2\inprod{X}{\ub}^2 ~\big|~\mathcal{E}\right] - \Exs\left[Y^2\inprod{X}{\ub}^2 ~\big|~ \mathcal{E}\right]\Prob(\mathcal{E}) -  \Exs\left[Y^2\inprod{X}{\ub}^2 ~\big|~ \mathcal{E}^c\right]\Prob(\mathcal{E}^c) \\
	& \leq \left|\Exs\left[Y^2\inprod{X}{\ub}^2 ~\big|~\mathcal{E}\right]\right|\Prob(\mathcal{E}^c) + \left|\Exs\left[Y^2\inprod{ X}{ \ub }^2 ~\big|~ \mathcal{E}^c\right]\Prob(\mathcal{E}^c)\right|. \\
	& \lesssim (a^2 + b^2)\left(\tau_1^3e^{-\tau_1^2/2} + \tau_1\tau_2 e^{-\tau_1^2/2-\tau_2^2/2}\right).
	\end{align*}
	\vskip .1in
	\noindent3. There exists $\ub \in \sphere^{p-1}$ such that
	\[
	\delta_3 := \opnorm{\Exs\left[Y^3X\otimes X \otimes X~\big|~\mathcal{E}\right] - \Exs\left[Y^3X\otimes X \otimes X\right]} = \Exs\left[Y^3\inprod{ X}{ \ub}^3 ~\big|~\mathcal{E}\right] - \Exs\left[Y^3\inprod{X}{\ub}^3\right].
	\]
	Using the same simplification argument in (a), we have
	\begin{align*}
	\Exs\left[Y^3\inprod{X}{ \ub}^3 ~\big|~\mathcal{E}\right] & = \Exs\left[(aX_1 + bZ)^3(cX_1 + dX_2)^3 ~\big|~\mathcal{E}\right] \\
	& = \Exs\left[a^3c^3X_1^6 + 3ab^2c^3X_1^4Z^2 + 3a^3cd^2X_1^4X_2^2 + 9ab^2cd^2X_1^2X_2^2Z^2 ~\big|~ \mathcal{E}\right] \lesssim a^3 + ab^2.
	\end{align*}
	Applying Lemma~\ref{lem:cond_moment} again leads to
	\begin{align*}
	\Exs\left[Y^3\inprod{X}{\ub}^3 ~\big|~\mathcal{E}^c\right] \Prob(\mathcal{E}^c) & \lesssim a^3c^3\tau_1^5e^{-\tau_1^2/2} + ab^2c^2\tau_1^3e^{-\tau_1^2/2} + a^3cd^2\tau_1^3\tau_2e^{-\tau_1^2/2-\tau_2^2/2} + ab^2cd^2\tau_1\tau_2e^{-\tau_1^2/2-\tau_2^2/2} \\
	& \lesssim (a^3 + ab^2)\left(\tau_1^5e^{-\tau_1^2/2} + \tau_1^3\tau_2e^{-\tau_1^2/2-\tau_2^2/2}\right).
	\end{align*}
	Finally, we have
	\begin{align*}
	\delta_3  & = \Exs\left[Y^3\inprod{ X}{\ub}^3 ~\big|~\mathcal{E}\right] - \Exs\left[Y^3\inprod{X}{\ub}^3 ~\big|~ \mathcal{E}\right]\Prob(\mathcal{E}) -  \Exs\left[Y^3\inprod{ X}{\ub}^3 ~\big|~ \mathcal{E}^c\right]\Prob(\mathcal{E}^c) \\
	& \leq \left|\Exs\left[Y^3\inprod{X}{\ub}^3 ~\big|~\mathcal{E}\right]\right|\Prob(\mathcal{E}^c) + \left|\Exs\left[Y^3\inprod{ X}{ \ub}^3 ~\big|~ \mathcal{E}^c\right]\Prob(\mathcal{E}^c)\right|. \\
	& \lesssim  (a^3 + ab^2)\left(\tau_1^5e^{-\tau_1^2/2} + \tau_1^3\tau_2e^{-\tau_1^2/2-\tau_2^2/2}\right).
	\end{align*}
\end{proof}

\begin{lemma} [Conditional Moments of Gaussian] \label{lem:cond_moment}Suppose $X \sim \mathcal{N}(0,1)$. For any $\tau > 0$ and positive integer $a$, we define
	\[
	m_a(\tau)  := \Exs\left[X^a ~\big|~ |X| > \tau\right]\Prob(|X| > \tau).
	\]
	Then we have that for all $a = 2,4,6,\dots$, we have
	\[
	m_a(\tau) = (a - 1)m_{a-2}(\tau) + \sqrt{\frac{2}{\pi}}\tau^{a-1}e^{-\frac{\tau^2}{2}}.
	\]
\end{lemma}
\begin{proof}
	The result follows from elementary calculation on Gaussian's probability density function. We omit the details.	
\end{proof}

\begin{lemma}[Sub-Gaussianity]  \label{lem:sub-gaussian}
	Let $X \sim \Gaussian(\bm{0}, \Identity_p)$. For any $k$ fixed vectors $\ub_1,...,\ub_k \in \mathbb{R}^{p}$, we define event
	\[
	\mathcal{E} := \left\{\abs{\inprod{X}{\ub_1}\rangle} \leq \abs{\inprod{X}{\ub_j}}, \;\;\text{for all}\;\; j \in [k]\right\}.
	\]
	
	\noindent \emph{(a)} Suppose $\Prob(\mathcal{E}) \geq \tau > 0$. There exists constant $C$ that only depends on $\tau$ such that for any fixed $\x \in \sphere^{p-1}$, we have that
	\[
	\Prob\left(\abs{\inprod{X}{\x}} > t ~\big|~ \mathcal{E} \right) \leq e^{1 - Ct^2},~~\text{for all}~~ t > 0.
	\]
	
	\noindent \emph{(b)} In general there exists constant $C'$ such that for any fixed $\x \in \sphere^{p-1}$,
	\[
	\Prob \left(\abs{\inprod{X}{\x}} > t ~\big|~ \mathcal{E} \right) \leq e^{1 - \frac{C'}{4k\log(k+1)}t^2}, ~~\text{for all}~~ t > 0.
	\] 
\end{lemma}
\begin{proof}
	Since $X$ is sub-Gaussian random vector, equivalently there exists constant $C$ such that for any fixed $\x \in \sphere^{p-1}$,
	\[
	\Prob \left\{\abs{\inprod{X}{\x}} \geq t\right\} \leq 1 \wedge e^{1 - Ct^2},~~\text{for all}~~ t > 0.
	\]
	\vskip .1in
	\noindent (a) Note that
	\begin{align*}
	\Prob \left(\abs{\inprod{X}{\x}} \geq t\right) & = \Prob \left(\abs{\inprod{X}{\x}} \geq t ~\big|~ \mathcal{E}\right)\Prob(\mathcal{E}) + \Prob \left(\abs{\inprod{X}{\x}} \geq t ~\big|~ \mathcal{E}^c\right)\Prob(\mathcal{E}^c) \\
	& \geq \tau \cdot \Prob \left(|\inprod{X}{\x}| \geq t ~\big|~ \mathcal{E}\right).
	\end{align*}
	Hence,
	\[
	\Prob \left(\abs{\inprod{X}{\x}} \geq t ~\big|~ \mathcal{E}\right) \leq 1 \wedge \tau^{-1}e^{1 - Ct^2} \leq 1 \wedge e^{1 - C'(\tau) t^2},
	\]
	where the last inequality holds for $C'(\tau) = C(1 - \log \tau)^{-1}$.
	\vskip .1in
	\noindent (b) Without loss of generality, we assume that $\ub_1,...,\ub_k$ live in the subspace spanned by $\e_1,...,\e_k$. For any vector $\ub \in \real^p$, we let $\ub_{[k]}$ be its sub-vector that contains the first $k$ coordinates, and $\ub_{\bot}$ be its sub-vector that contains the rest coordinates. For any $\x \in \mathbb{S}^{p-1}$, we have
	\begin{align} \label{eq:tmp9}
	\Prob\left(\abs{\inprod{X}{\x}} > t ~\big|~ \mathcal{E}\right) & \leq \Prob\left(\abs{\inprod{X_{[k]}}{\x_{[k]}}} > t/2 ~\big|~ \mathcal{E}\right) + \Prob\left(\abs{\inprod{ X_{\bot}}{\x_{\bot}}} > t/2\right) \notag\\
	& \leq \Prob\left(\abs{\inprod{ X_{[k]}}{\x_{[k]}}} > t/2 ~\big|~ \mathcal{E}\right) + e^{1 - Ct^2/4} \notag\\
	& \leq \Prob\left(\twonorm{X_{[k]}} > t/2 ~\big|~ \mathcal{E}\right) + e^{1 - C t^2/4}. 
	\end{align}
	Note that conditioning $\mathcal{E}$ does not change the distribution of $\twonorm{X_{[k]}}$. We thus have
	\[
	\Prob\left(\twonorm{X_{[k]}} > t/2 ~\big|~ \mathcal{E}\right) = \Prob\left(\twonorm{X_{[k]}} > t/2 \right) \leq \sum_{i \in [k]}\Prob\left(|X_i| \geq \frac{t}{2\sqrt{k}}\right) \leq k\cdot e^{1-Ct^2/(4k)}.
	\]
	Combining \eqref{eq:tmp9} with the above inequality yields that
	\begin{align*}
	\Prob\left(\abs{\inprod{X}{\x}} > t ~\big|~ \mathcal{E}\right) \leq  1 \wedge (k+1)e^{1 - Ct^2/(4k)} \leq 1 \wedge e^{1 - C_2(k)t^2},
	\end{align*}
	where the last inequality holds by setting $C_2(k) = \frac{C}{4k\log(k+1)}$. 
\end{proof}

\printbibliography

\end{document}